\documentclass{article}

\usepackage{microtype}
\usepackage{graphicx}
\usepackage{subfigure}
\usepackage{booktabs, colortbl, siunitx}
\usepackage{multirow}

\usepackage{algorithm}
\usepackage[noend]{algorithmic}

\newcommand{\methname}{{TransFusion}}
\newcommand{\tit}[1]{\smallbreak\noindent\textbf{#1 }}

\newcommand{\ie}{\textit{i.e.}}

\usepackage{amsmath}
\usepackage{amssymb}
\usepackage{mathtools}
\usepackage{amsthm}
\usepackage{tcolorbox}
\usepackage{hyperref}
\usepackage{cancel}
\usepackage{wrapfig}
\usepackage[capitalize,noabbrev]{cleveref}

\crefname{section}{Sec.}{Sec.}
\crefname{table}{Tab.}{Tab.}
\crefname{figure}{Fig.}{Fig.}
\crefname{equation}{Eq.}{Eq.}

\usepackage[accepted]{icml2025}
\usepackage{caption}

\theoremstyle{plain}
\newtheorem{theorem}{Theorem}[section]
\newtheorem{proposition}[theorem]{Proposition}

\theoremstyle{definition}

\theoremstyle{remark}

\usepackage[textsize=tiny]{todonotes}

\definecolor{LightCyan}{rgb}{0.88,1,1}
\definecolor{Gray}{gray}{0.94}

\newcolumntype{g}{>{\columncolor{Gray}}c}
\newcolumntype{a}{>{\columncolor{LightCyan}}c}

\icmltitlerunning{Update Your Transformer to the Latest Release: Re-Basin of Task Vectors}
\begin{document}

\twocolumn[
\icmltitle{Update Your Transformer to the Latest Release: Re-Basin of Task Vectors}

\icmlsetsymbol{equal}{*}

\begin{icmlauthorlist}
\icmlauthor{Filippo Rinaldi}{equal,unimore}
\icmlauthor{Giacomo Capitani}{equal,unimore}
\icmlauthor{Lorenzo Bonicelli}{unimore}
\icmlauthor{Donato Crisostomi}{sapienza}
\icmlauthor{Federico Bolelli}{unimore}
\vspace{0.5em}

\icmlauthor{Elisa Ficarra}{unimore}
\icmlauthor{Emanuele Rodolà}{sapienza}
\icmlauthor{Simone Calderara}{unimore}
\icmlauthor{Angelo Porrello}{unimore}
\end{icmlauthorlist}

\icmlaffiliation{unimore}{AImageLab, University of Modena and Reggio Emilia, Italy.}
\icmlaffiliation{sapienza}{Sapienza, University of Rome, Italy}

\icmlcorrespondingauthor{Filippo Rinaldi}{filippo.rinaldi@unimore.it}

\icmlkeywords{Machine Learning, ICML, Weight Interpolation, Transfer Learning, Model Re-basin, Model Editing, Model Patching}
\vskip 0.3in]

\printAffiliationsAndNotice{\icmlEqualContribution} %

\begin{abstract}
Foundation models serve as the backbone for numerous specialized models developed through fine-tuning. However, when the underlying pretrained model is updated or retrained (e.g., on larger and more curated datasets), the fine-tuned model becomes obsolete, losing its utility and requiring retraining. This raises the question: is it possible to transfer fine-tuning to a new release of the model? In this work, we investigate how to transfer fine-tuning to a new checkpoint without having to re-train, in a data-free manner. To do so, we draw principles from model re-basin and provide a recipe based on weight permutations to re-base the modifications made to the original base model, often called task vector. In particular, our approach tailors model re-basin for Transformer models, taking into account the challenges of residual connections and multi-head attention layers. Specifically, we propose a two-level method rooted in spectral theory, initially permuting the attention heads and subsequently adjusting parameters within select pairs of heads. Through extensive experiments on visual and textual tasks, we achieve the seamless transfer of fine-tuned knowledge to new pre-trained backbones without relying on a single training step or datapoint. Code is available at \url{https://github.com/aimagelab/TransFusion}.
\end{abstract}

\section{Introduction}
Recently, there has been a notable shift among researchers and practitioners towards fine-tuning pre-trained models, rather than building them from scratch. This method leverages backbones developed on large-scale datasets, considerably decreasing the amount of data and training time needed to tailor models for specific downstream tasks. For this reason, pre-trained backbones such as OpenAI's CLIP~\cite{radford2021learning} are being extensively utilized as base foundation models. As a result, the corresponding fine-tuned versions play a crucial role in numerous real-world applications like medical imaging~\cite{lu2024avisionlanguage} and satellite image analysis~\cite{mall2024remote}.
\begin{figure}[t]
\vskip 0.2in
\centering
\includegraphics[width=0.99\linewidth]{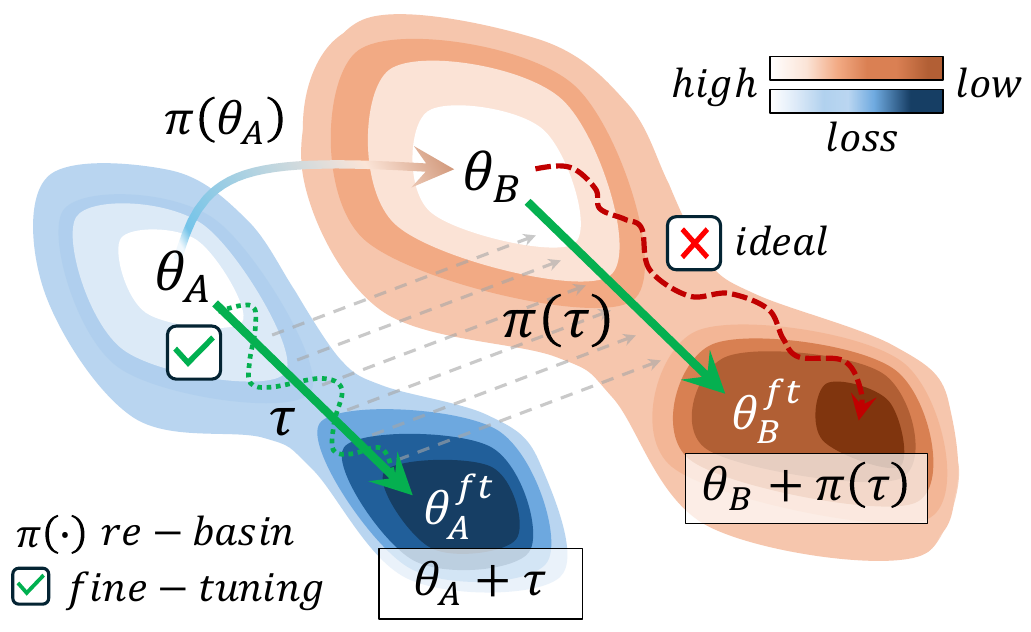}
\caption{Transporting task vector $\tau$ from a fine-tuned base model $\theta_A^{ft} = \theta_A + \tau$ to a new release $\theta_B$.}
\label{fig:intro}
\vskip -0.2in
\end{figure}

However, while these pre-trained backbones are widely adopted, their evolution poses new challenges, with tech companies and academic institutions frequently releasing updated checkpoints. Often, these updates do not modify the underlying architecture but simply consist of new weights, trained on increasingly large datasets compared to their predecessors~\cite{ilharco_gabriel_2021_5143773}. Moreover, the additional training data may be more curated or specifically tailored to specialized domains, boosting their zero-shot capabilities considerably.

To take advantage of newly released checkpoints, the typical approach is to retrain them on the downstream task. This means fine-tuning the new checkpoint on the same data already used to adapt the original model. Besides the considerable costs associated with re-training the new model, this strategy is also unviable in certain scenarios. Indeed, the data for the downstream task might be no longer available due to compliance with privacy or storage constraints. 

This raises an important question: can we \textbf{re-use} the fine-tuning that has already been performed on the newly released model? Precisely, the overall aim of this paper is to investigate whether we can \textit{transport} the previous fine-tuning, in a training-free manner. To understand the idea of \textit{transport}, we consider the weights of the original base model as $\theta_{A}$, and their fine-tuning as $\theta_{A}^{ft} = \theta_{A} + \tau$. The task vector~\cite{ilharcoediting, ortiz2024task} $\tau = \theta_{A}^{ft} - \theta_{A}$ represents a direction from $\theta_{A}$ that embodies all adjustments made during the fine-tuning process. Hence, our goal is to find a procedure $\pi(\cdot)$ that can transport the task vector $\tau$ into an appealing basin of the newly released model $\theta_{B}$ (see~\cref{fig:intro}). In this approach, the procedure $\pi(\cdot)$ must be designed to ensure that the modified weights $\theta_{B}^{ft} = \theta_{B} + \pi(\tau)$ achieve low loss on the downstream task.

When designing the transportation function $\pi(\cdot)$, the ideal approach should be data-free and training-free to meet the concerns above. Nevertheless, if the two base models $\theta_A$ and $\theta_B$ differ significantly (due to varying initialization, training strategies, or datasets), the knowledge acquired during fine-tuning of $\theta_A$ may not transfer to $\theta_B$ with a mere addition of the original task vector (\ie, $\theta_{B}^{ft} = \theta_{B} + \tau$). To bridge the gap in representation spaces and facilitate the transfer, intuitively, we have to make the two base models ``compatible'', such that they ``speak the same language''. To address this challenge, we could \textit{rebase} one of the two models (for instance, $\theta_A$), such that any linear interpolation between the weights of the edited $\theta^{\prime}_{A}$ and $\theta_{B}$ yields an intermediate model that performs comparably to both $\theta_{A}$ and $\theta_{B}$. This indicates that the models are now aligned and thus share a common low-loss basin. Notably, this concept of re-basin models shares similarities with the approach described in~\cite{ainsworth2023git}, where alignment is achieved by finding optimal permutations of the rows in the weight matrices. In this paper, we build upon this idea and explore its application in the context of fine-tuning, with a task vector being permuted and finally applied to $\theta_{B}$.

While model re-basin presents an appealing framework, it currently faces several technical hindrances. To date, successful applications of model re-basin have been limited to Multi-Layer Perceptrons (MLPs) and Convolutional Neural Networks (CNNs)~\cite{ainsworth2023git}. Unfortunately, the application of model re-basin to multi-head attention layers, despite their widespread use in Transformer-based architectures, has been largely overlooked, with a few very recent attempts based on Optimal Transport (see \cref{sec:related}). However, these recent methods do not guarantee \textbf{functional equivalence} between the permuted model and the original, unpermuted model. The primary obstacle lies in managing the weights associated with multiple attention heads. As we discuss in \cref{sec:method}, indeed, to apply standard permutation-based approaches~\cite{ainsworth2023git, singh2020model, imfeld2024transformer}, the heads must be concatenated and treated as a single unified projection. This way, after applying permutations, each head of the edited model may incorporate rows from different original heads — an issue we refer to as \textit{head contamination}. This is problematic because, without preserving the logical separation of heads during permutation, it becomes impossible to invert the permutation process and recover the original, unpermuted output of the attention block. Furthermore, existing methods struggle when two addends in the computational graph rely on distinct permutation matrices, a situation common in residual connections such as $h' = h + f(h)$. Differently, we avoid averaging the respective permutation matrices, thereby preserving their discrete nature.

To address these issues, we propose a structured two-level approach for effective re-basin of Transformer-based models, called \methname{}. To avoid \textit{head contamination}, \methname{} firstly seeks optimal mappings between pairs of heads (\textit{inter-head} permutations); subsequently, we restrict permutations to only the rows within these coupled heads (\textit{intra-head} permutations).We mathematically prove that this two-level permutation strategy prevents head contamination and preserves \textbf{functional equivalence} between the original and permuted models. Notably, the \textit{inter-head} permutations are optimized leveraging a distance metric that is \textit{invariant} to permutations of the rows and columns within the heads. Such a metric is founded on spectral theory~\cite{jovanovic2012spectral} and employs the singular values of the weight matrices, which are unaffected by orthogonal transformations like those induced by permutations. 

We show that transporting task vectors allows the transfer of knowledge into a new checkpoint in a data-free manner. In practice, this means we can improve the zero-shot performance of the new version on the downstream task. We also demonstrate that the transport retains the generalization capabilities on a support set --- a crucial factor to justify updating the base model to the new release. To summarize: 

\begin{tcolorbox}[colback=gray!2!white, colframe=black]
\textbf{Motivation.} Given a fine-tuned Transformer model $f(\cdot, \theta_{A}^{ft} = \theta_{A} + \tau)$ we aim to transfer the task vector $\tau$ into a new release $f(\cdot, \theta_{B})$ in a data-free manner.

\smallskip
\textbf{Contributions.} \textit{i)} We introduce a novel data-free method for aligning Transformers. \textit{ii)} We propose a permutation-invariant spectral measure to manage multi-head attention layers. \textit{iii)} We reveal the practicality of applying a task vector from one base model $\theta_A$ to a new release $\theta_B$.
\end{tcolorbox}

\section{Background}
\label{sec:background}
Given two models with weights $\theta_A$ and $\theta_B$, \textbf{model re-basin}~\cite{ainsworth2023git} investigates how to permute the units of one model to facilitate the alignment of two models. The two models are then merged in the weight space, resulting in an interpolated model that achieves performance comparable to that of the two original ones. 

Following the notation of~\citep{ainsworth2023git}, \textbf{re-basin} is defined as any operation defined on the weights of two models $\theta_A$ and $\theta_B$ that maps one of the two models onto the local loss region (\textit{basin}) of the other one. To assess the effectiveness of re-basin, one common approach is to check the property of linear mode connectivity~\cite{frankle2020linear,entezari2022the} between the permuted model and the other, reference model. Informally, this involves checking if the model weights laying on the linear path connecting $\theta_A$ and $\theta_B$ also result in a low loss value.

To reach such a property, existing model re-basin techniques leverage the permutation symmetries inherent in neural networks~\cite{entezari2022the}. These symmetries allow the swapping of the units within a layer without changing the functionality of the network. To show that, we consider the activation of the $\ell$-th layer of an MLP:
\begin{equation}
z_{\ell+1} = \sigma(W_{\ell} z_{\ell} + b_{\ell}), \quad z_0 = x,
\end{equation}
where $W_{\ell}$ and $b_{\ell}$ are the weight matrix and bias vector and $\sigma$ denotes an element-wise activation function. In this case, the following relation holds for any permutation matrix $P$:
\begin{align}
\label{eq:rebasinone}
z_{\ell+1} &= P^\top P z_{\ell+1} = P^\top P \sigma(W_{\ell} z_{\ell} + b_{\ell}), \\
&= P^\top \sigma(P W_{\ell} z_{\ell} + P b_{\ell}), \quad \text{where } P \in S_d,
\end{align}
with $S_d$ denoting the set of $d \times d$ permutation matrices. Thanks to this relation, we can essentially permute the weights and biases of a layer using a matrix $P$. Therefore, when we apply the permutation $P$ to the parameters of a layer, the resulting output undergoes the same permutation. However, to ensure that the transformed model remains functionally equivalent to the original, the next layer must process the output in its original, unpermuted form. This can be achieved equivalently by permuting the weights of the subsequent layer using the inverse permutation $P^\top$. Accordingly, we define a transformed set of weights $\theta'$ as:
\begin{equation}
\label{eq:rebasintwo}
W_{\ell}' = P W_{\ell}, \quad b_{\ell}' = P b_{\ell}, \quad W_{\ell+1}' = W_{\ell+1} P^\top.
\end{equation}
\tit{Git Re-Basin.} \Citet{ainsworth2023git} exploit \cref{eq:rebasinone} to induce weight alignment between $\theta_A$ and $\theta_B$. Formally, we consider the $\ell$-th feed-forward layer, with weight matrices $W^{(A)}_{\ell}$ and $W^{(B)}_{\ell}$ for $\theta_A$ and $\theta_B$ respectively. Given that each row of $W^{(A)}_{\ell}$ and $W^{(B)}_{\ell}$ represents a distinct feature, if $[W^{(A)}_{\ell}]_{i,:} \approx [W^{(B)}_{\ell}]_{j,:}$, then it makes sense to associate the units $i$ and $j$. Therefore, we could formalize the alignment as finding the permutation matrix that maximizes the dot product between $P_\ell W^{(A)}_{\ell}$ and $W^{(B)}_{\ell}$. However, to preserve functional equivalence, we have to account for the term $P_{\ell-1}^{\top}$ related to the permutation of the previous layer ---see \cref{eq:rebasintwo}. This results in a global optimization across layers: 
\begin{align}
\label{eq:optproblem}
\underset{\pi=\left\{P_{\ell}\right\}_{1}^{L}} {\operatorname{argmax}} & \left\langle W^{(B)}_{1}, P_{1} W^{(A)}_{1} \right\rangle + \left\langle W^{(B)}_{2}, P_{2} W^{(A)}_{2} P_{1}^\top \right\rangle + \nonumber \\ 
& \ \ \ \ \ \ \ + \dots + \left\langle W^{(B)}_{L}, W^{(A)}_{L} P_{L-1}^{\top} \right\rangle.
\end{align}
where $\langle A, B \rangle = \sum_{i,j} A_{i,j} B_{i,j}$ is the inner product between real-valued matrices. As discussed in~\cite{ainsworth2023git}, the optimization problem described in \cref{eq:optproblem} corresponds to the Symmetric Orthogonal Bilinear Assignment Problem (SOBLAP), which is unfortunately NP-hard. Its relaxation re-casts it as a series of Linear Assignment Problems (LAPs), focusing on one permutation $\boldsymbol{P_\ell}$ at a time while keeping the others fixed. In formal terms:
\begin{equation}
\begin{aligned}
\underset{\textcolor{red}{P_{\ell}}}{\operatorname{argmax}} &\left\langle W_{\ell}^{(B)}, \textcolor{red}{P_{\ell}} W_{\ell}^{(A)} P_{\ell-1}^{\top}\right\rangle
+ \left\langle W_{\ell+1}^{(B)}, P_{\ell+1} W_{\ell+1}^{(A)} \textcolor{red}{P_{\ell}}^{\top} \right\rangle.
\end{aligned}
\label{eq:ff-rebasin}
\end{equation}
Notably, each LAP can be solved with efficient, polynomial-time methods like the Hungarian algorithm~\cite{jonker1988shortest}. The outcome is a set of permutation matrices $\pi=\left\{P_{\ell}\right\}_{1}^{L}$, which, when applied to model $\theta_A$, result in a new model $\theta_{A'} = \pi(\theta_A)$. Notably, this model is functionally equivalent and, theoretically, it resides in the low-loss basin of $\theta_B$. However, as optimizing a series of LAPs is a coarse approximation of the SOBLAP, there are no strong guarantees regarding the optimality of the solution.

\section{\methname{}: An Approach For Re-basin Transformer Models}
\label{sec:method}
\tit{Objective.} Our approach, named \methname{}, is designed to transfer task-specific knowledge between transformer-based models that have undergone different pre-training. Specifically, it starts with an initial weight set $\theta_A$ and a task vector $\tau = \theta_A^{ft} - \theta_A$, derived after fine-tuning on a downstream task. The goal is to adapt $\tau$ to a new parameter configuration $\theta_B$. This process aims to preserve the inherent properties of $\theta_B$ --- for example, its superior zero-shot capabilities compared to $\theta_A$ --- and to integrate specialized knowledge carried out by $\tau$ for the downstream task. Finally, we aim to enable model transfer in a data-free manner.

\tit{Weight Matching.} To achieve these objectives, we start by aligning the weights of $\theta_A$ with those of $\theta_B$. This is accomplished with a novel \textit{data-free weight matching strategy} tailored for Transformer architectures. The procedure is deeply discussed in \cref{sec:weigthmatching}. Briefly, we tackle various shortcomings of existing methods and handle two building blocks of attention-based networks, namely the residual paths and the multi-head attention mechanism. To manage the latter, we introduce a novel two-step process that employs a permutation-invariant spectral metric to match pairs of heads within the same layer of $\theta_A$ and $\theta_B$. Subsequently, we permute features within the matched heads to optimize weight alignment, as detailed in \cref{sec:background}.

\tit{Transport.} We end up with a functionally equivalent model $\theta_A' = \pi(\theta_A)$, where $\pi(\cdot)$ yields a permutation of every layer in $\theta_A$. Afterwards, $\pi(\cdot)$ is used to transport the task vector $\tau = \theta_{ft} - \theta_A$ into the low-loss basin of $\theta_B$ (see~\cref{sec:transportation}).
\subsection{Attention Alignment for Transformer Models}
\label{sec:weigthmatching}
A Transformer-based block consists of a multi-head attention layer and an MLP block, connected through residual connections. Considering the MLP, this builds upon standard linear projections, which we treat as discussed in~\cref{sec:background}. Instead, we adopt a novel, tailored approach for multi-head attention layers addressing a common pitfall. Considering multiple heads, current methods view their projections as a whole linear layer, thereby joining the corresponding weight matrices before applying permutations. However, such an approach does not reflect the organization of attention in distinct, parallel heads. For example, this can result in artifacts, where units from separate heads in the original model are mixed together --- an issue we call \textit{head contamination}. This compromises the structural separability of attention heads and precludes the preservation of \textit{functional equivalence}, that is, the ability to permute and subsequently unpermute the weight matrices while yielding identical model outputs.

In the following, we present our proposal against head contamination (see \textbf{Step 1} and \textbf{2}) and a practical approach to handle residual connections (\textbf{Step 3}). The complete methodology is outlined in~\cref{global_algo}. 

\tit{Step 1: Inter-Head Alignment}
\\
Consider the q(uery), k(ey), and v(alue) projection matrices $W_q$, $W_k$, and $W_v$ $\in \mathbb{R}^{d_m \times d_m}$ --- 
with $d_m$ denoting the total embedding dimension of the attention module. We partition each matrix into $H = \#\text{heads}$ matrices (one for each head) of shape $d_k \times d_m$, where $d_k = \frac{d_m}{H}$. This results in a tensor $\Tilde{W}_q = \operatorname{split}(W_q, H) \in \mathbb{R}^{H \times d_k \times d_m}$ for the query projection matrix $W_q$. The same operation is applied for $W_k$ and $W_v$ to obtain $\Tilde{W}_k$ and $\Tilde{W}_v$.

\begin{figure}[t]
\vskip 0.2in
\centering
\includegraphics[width=0.48\textwidth]{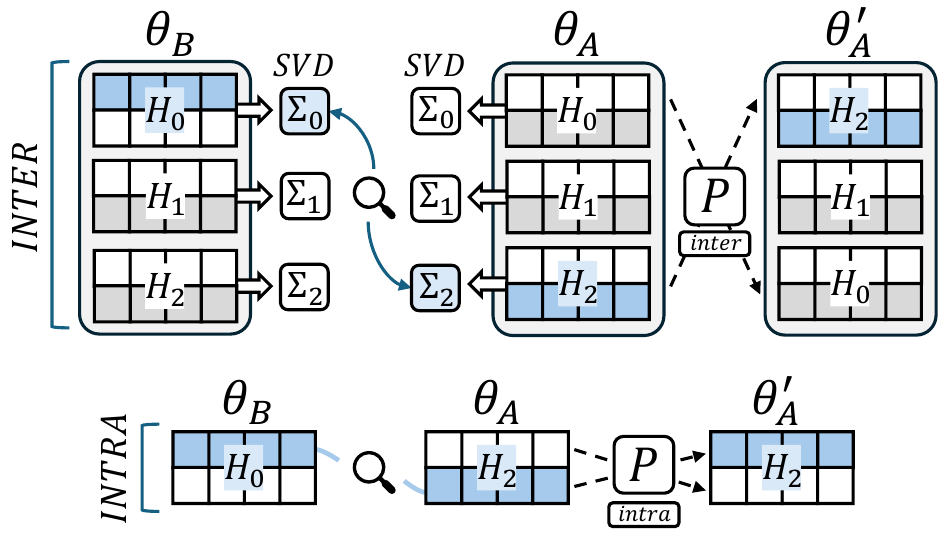}
\caption{Inter- (Step 1) and intra-head alignment (Step 2).}
\label{fig:heads}
\vskip -0.2in
\end{figure}

The first step involves defining a distance metric between pairs of heads, such that we can identify and execute the optimal swap between heads in $\theta_A$ and $\theta_B$ (see~\cref{fig:heads}). We employ a distance metric that is \textbf{invariant} to permutations of rows and columns within the $H$ sub-matrices in $\Tilde{W}_q$, $\Tilde{W}_k$, and $\Tilde{W}_v$. In this respect, one might question why invariance is crucial for comparisons between different heads. We note that the initial, natural order of units does not necessarily correspond to the optimal alignment that could be achieved. Consequently, the metric used in this initial phase must be insensitive to the specific ordering of head features, thereby ensuring an agnostic comparison of the heads.

To achieve the permutation-invariance property, we employ a distance based on the \textbf{singular values} of the sub-matrices representing the heads. Specifically, given two heads $h_i^B = [\Tilde{W}]_{i,:,:}^B \in \mathbb{R}^{d_k \times d_m}$ from model $\theta_B$ and $h_j^A = [\Tilde{W}]_{j,:,:}^A$ from model $\theta_A$, we compute the distance as:
\begin{equation}
\label{eq:distance_element}
d_{ij} = \left\| \Sigma_{i} - \Sigma_{j} \right\|,
\end{equation}
where $\Sigma_{i}$ and $\Sigma_{j}$ denote the singular values of $h_i^B$ and $h_j^A$ respectively. These can be computed through the Singular Value Decomposition (SVD); in formal terms, considering the $i$-th head, the SVD decompose its weight $h_i^B = U_i \Sigma_i V_i^T$, where $U_i$ and $V_i$ are orthogonal matrices, and $\Sigma_i$ is a diagonal matrix containing the singular values of $h_i$. As demonstrated in~\cref{sec:teorema}, the Euclidean distance between singular values remains invariant to permutations.

To take into account the distance for query, key and value projections jointly, we construct a distance matrix $D \in \mathbb{R}^{H \times H}$, where each element $D_{ij} = d_{ij}^q + d_{ij}^k + d_{ij}^v$ represents an inter-head alignment cost that is calculated as the sum of the pairwise distances across $q$, $k$ and $v$ matrices. We hence employ $D$ to find the optimal inter-head permutation:
\begin{equation}
P_{\text{inter\_head}} = \underset{P \in S_H}{\operatorname{argmin}}\sum_{i=1}^{H} D_{i, P[i]}, \label{eq:distance_matrix}
\end{equation}
where $D_{i, P[i]}$ is the distance between the $i$-th head of model $\theta_B$ and the $P[i]$-th (candidate) head of model $\theta_A$. The solution $P_{\text{inter\_head}}$ can be practically determined with the Hungarian algorithm. The corresponding permutation is applied to each $\Tilde{W}_q^A$, $\Tilde{W}_k^A$, and $\Tilde{W}_v^A$, thereby reordering the heads of $\theta_A$ to increase alignment with those of $\theta_B$.
\tit{Step 2: Intra-Head Alignment}\\
After matching each pair of heads $(h_i^B, h_{P[i]}^A)$, we aim to swap individual units between $h_i^B$ and $h_{P[i]}^A$. To do so, as in Git Re-Basin \cite{ainsworth2023git}, we seek for permutations that maximize the inner product between the corresponding $H$ sub-portions of the projection weights $h_i^B$ and $h_{P[i]}^A$, as follows:
\begin{align}
\label{eq:intra_head_obj}
P_{\text{intra\_head}} &= \{ P_{\text{intra\_head}}^{(i)}\}_{i=1}^{H}, \nonumber\\
\text{where} \  P_{\text{intra\_head}}^{(i)} &=\underset{P \in S_{d_k}}{\operatorname{argmax}} \left\langle h_i^B, P h_{P[i]}^A \right\rangle. 
\end{align}
In this formula, the cost considers the dot products across query, key, and value sub-matrices. In the end, the optimization yields $H$ permutation matrices $P_{\text{intra\_head}}^{(i)} \in \mathcal{S}_{d_k}$.

In summary, this two-step process ensures both global inter-head and local intra-head alignment while preserving the structural integrity of the multi-head attention mechanism. 
The comprehensive procedure for aligning multi-head attention layers, combining these alignment stages, is formalized in~\cref{mha_algo}. Specifically, these individual alignment steps are unified into a single composed permutation, denoted as $P_{\text{attn}}$, which is applied directly to the projection matrices of the multi-head attention layer. Crucially, it can be shown that our structured composed permutation preserves functional equivalence: despite reordering and permuting the heads and their internal dimensions, the attention computation remains unchanged (up to a corresponding permutation of its outputs). We formalize this in the following theorem.

\begin{theorem}[\textbf{Equivariance of Multi-Head Attention to Structured Permutations}]
\label{thm:attn-equiv}
Let $P_{\text{inter\_head}} \in S_H$ be a permutation over the $H$ attention heads, and let $P_{\text{intra\_head}} = \{ P_{\text{intra\_head}}^{(i)} \}_{i=1}^{H}$ be a set of independent permutations acting within each head (of size $d_k = \tfrac{d_m}{H}$). Then applying the composed block permutation $P_{\text{attn}}$ to each of the projection matrices $W_q, W_k, W_v \in \mathbb{R}^{d_m \times d_m}$ is functionally equivalent to permuting the output of the multi-head attention module. The resulting attention output $O'$ satisfies: $O' = O P_{\text{attn}} $, where $O$ is the unpermuted output.
\end{theorem}

A complete proof of~\cref{thm:attn-equiv} is provided in~\cref{app:proof-attn-equiv}.

\newcommand{\algorithmname}[1]{\text{Multi-Head Attention Alignment (Alg. \ref{mha_algo})}}

\newlength\myindent
\setlength\myindent{1em}
\newcommand\bindent{%
  \begingroup
  \setlength{\itemindent}{\myindent}
  \addtolength{\algorithmicindent}{\myindent}
}
\newcommand\eindent{\endgroup}
\begin{algorithm}[t]
    \caption{Weight Matching}
    \label{global_algo}
    \begin{algorithmic}[1]
        \REQUIRE $\theta_A=\left\{W_{\ell}^{(A)}\right\}_{\ell=1}^{L}$ and $\theta_B=\left\{W_{\ell}^{(B)}\right\}_{\ell=1}^{L}$
        \smallskip
        \ENSURE A permutation $\pi = \left\{P_1, \ldots, P_{L-1}\right\}$ of $\theta_A$. %
        \smallskip
        \STATE \textbf{Initialize}: $P_1 \leftarrow I, \ldots, P_{L-1} \leftarrow I$
        \REPEAT
            \FOR{$\ell \in 1, \ldots, L-1$}
                \IF{$\ell$ is a MHA layer}
                    \STATE $P_{\ell} \leftarrow \algorithmname{mha_algo}$
                \ELSE
                    \STATE $P_{\ell} \leftarrow \text{Solving LAP as in \cref{eq:ff-rebasin}}$
                \ENDIF
            \ENDFOR
        \UNTIL{convergence}
    \end{algorithmic}
\end{algorithm}
\begin{algorithm}[t]
    \caption{Attention Alignment}
    \label{mha_algo}
    \begin{algorithmic}[1]
        \REQUIRE Weights $\Tilde{W}^{(A)}_q P_{\ell-1}^{\top}$, $\Tilde{W}^{(A)}_k P_{\ell-1}^{\top}$, $\Tilde{W}^{(A)}_v P_{\ell-1}^{\top} \in \theta_A$ and $\Tilde{W}^{(B)}_q$, $\Tilde{W}^{(B)}_k$, $\Tilde{W}^{(B)}_v \in \theta_B$ for multi-head attention projection layer $\ell$ and previous layer $\ell-1$.
        \smallskip
        \ENSURE Permutation $P_{\ell\text{-attn}}$ for $\Tilde{W}^{(A)}_q$, $\Tilde{W}^{(A)}_k$, $\Tilde{W}^{(A)}_v$.
        \smallskip

        \STATE \textbf{\textcolor{blue}{Step 1: Inter-Head Alignment}}
        \STATE Create spectral distance matrix $D$ (\cref{eq:distance_element}).
        \STATE $P_{\text{inter}} \leftarrow$ Solve LAP on $D$ (\cref{eq:distance_matrix}).

        \smallskip
        \STATE \textbf{\textcolor{blue}{Step 2: Intra-Head Alignment}}
        \FOR{$h = 1$ \textbf{to} $H$ head pairs from $P_{\text{inter}}$ }
            \STATE $P_{\text{intra}}^{(h)} \leftarrow$ Solve LAP for head pair $h$ (\cref{eq:intra_head_obj}).
        \ENDFOR

        \smallskip
        \STATE $P_{\ell\text{-attn}} \leftarrow P_{\text{inter}} \circ \big\{P_{\text{intra}}^{(h)}\big\}_{h=1}^H$ \ \ \ $\triangleright$ \text{compose permutations}
    \end{algorithmic}
\end{algorithm}

\tit{Step 3: Managing of Residual Connections} \\
Each transformer block incorporates two residual connections: the first bypasses the multi-head attention layer, and the second bypasses the feed-forward network:
\begin{equation}
\label{eq:resblock}
\begin{aligned}
\mathbf{z}_{attn} &= W_{0} \operatorname{MHA}(\mathbf{x}), \\
\mathbf{z}_i &= \mathbf{z}_{attn} + \mathbf{x}, \\ 
\mathbf{z}_f &= W_2 \text{ReLU}(W_1 \mathbf{z}_i), \\
\mathbf{z}_{out} &= \mathbf{z}_f + \mathbf{z}_i,
\end{aligned}
\end{equation}
where $\mathbf{x}$ is the input, $W_0$ is the weight of the attention mechanism, and $W_1$ and $W_2$ those of the feed-forward layer. For simplicity, we omit layer normalization as it can be regarded as a standard linear projection.

In each residual block, the input and the intermediate output are summed to produce the final output. However, we note that there are several sources of potential mismatch between the two addends: intuitively, if the two addends have undergone different permutations, it is reasonable to suspect a potential mismatch in their representations.

\setlength{\columnsep}{16pt}
\begin{wrapfigure}{r}{0.15\textwidth} %
\centering
\includegraphics[width=0.14\textwidth]{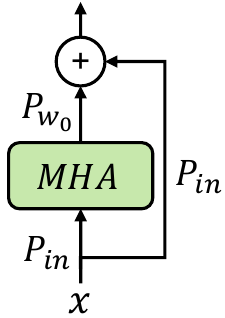}
\caption{Residuals block and permutations.}
\label{fig:residual}
\end{wrapfigure}

\definecolor{darkgreen}{RGB}{0, 150, 0}
\newcommand{\colheader}[1]{\multicolumn{2}{c}{\textbf{#1}}}
\newcommand{\metric}[1]{\textcolor{gray}{\textbf{#1}}}
\newcommand{\gain}[1]{\textcolor{darkgreen}{\textbf{#1}}}
\newcommand{\loss}[1]{\textcolor{gray}{\textbf{#1}}}
\newcommand{\dataset}[1]{[\textcolor{red}{\textbf{#1}}]}

\begin{table*}[t]
    \captionsetup{justification=centering}
    \caption{Comparison of permutation-based methods on visual tasks, in terms of task accuracy $[\uparrow]$ and support accuracy $[\uparrow]$.}
    \label{tab:gain_visual}
    \begin{center}
    \begin{small}
    \begin{sc}
    \begin{tabular}{l*{4}{>{}g>{}c}}
        \toprule
        \multirow{2}{*}{\textit{Method}} 
        & \multicolumn{2}{c}{\textbf{EuroSAT}} 
        & \multicolumn{2}{c}{\textbf{DTD}} 
        & \multicolumn{2}{c}{\textbf{GTSRB}} 
        & \multicolumn{2}{c}{\textbf{SVHN}} \\
        & Task & Supp. & Task & Supp. & Task & Supp. & Task & Supp. \\
        \midrule
        $\theta_{B}$ \textit{zero-shot}                         
        & $49.02$ & $68.73$ 
        & $47.50$ & $68.73$ 
        & $43.42$ & $68.73$ 
        & $45.97$ & $68.73$ \\
        
        $\theta_{B}$ + $\mathcal{\tau}$                         
        & -$7.62$ & -$16.15$ 
        & -$0.15$ & -$0.10$ 
        & -$5.39$ & -$0.70$ 
        & -$22.00$ & -$16.45$ \\
        
        $\theta_{B}$ + $\pi(\tau)$ \textit{(Optimal Transport)} 
        & -$14.05$ & -$5.28$ 
        & -$0.53$ & -$1.18$ 
        & -$2.43$ & -$1.30$ 
        & -$12.30$ & -$2.70$ \\
        
        $\theta_{B}$ + $\pi(\tau)$ \textit{(GiT Re-Basin)}       
        & +$0.95$ & -$0.48$ 
        & -$0.91$ & -$\mathbf{0.02}$ 
        & +$0.76$ & -$\mathbf{0.05}$ 
        & +$0.79$ & +$\mathbf{0.30}$ \\
        
        \midrule
        \textbf{TransFusion (Ours)}                             
        & +$\mathbf{4.95}$ & -$\mathbf{0.06}$ 
        & +$\mathbf{0.21}$ & -$0.08$ 
        & +$\mathbf{1.10}$ & -$0.40$ 
        & +$\mathbf{3.64}$ & -$0.48$ \\
        \bottomrule
    \end{tabular}
    \end{sc}
    \end{small}
    \end{center}
    \vskip -0.1in
\end{table*}

To clarify the interaction between permutations in residual blocks, consider~\cref{fig:residual}, which represents the first residual $\mathbf{z}_i = \mathbf{z}_{attn} + \mathbf{x}$. When the weights of the model are permuted, the input $\mathbf{x}$ comes with its own permutation $P_{\text{in}}$, which has to be accounted using $P_{\text{in}}^\top$. Moreover, the attention projection $W_0$ adds its own permutation matrix $P_{W_0}$. This leads to the following relation:
\begin{equation}
\mathbf{z}_i = P_{W_0} \mathbf{z}_{\text{attn}} + P_{\text{in}}\mathbf{x}.
\end{equation}
When examining the permutations/summations that impact on $\mathbf{z}_i$, there are two main issues: \textit{issue I)} the residual branch lacks transformations that could account for the matrix $P_{\text{in}}$; \textit{issue II)} the projection $W_{0}$ adds its own permutation $P_{W_{0}}$ of which the residual branch has no information about.

To maintain coherence between the two addends, they must be transformed under identical permutations. To enforce this consistency, we redefine the identity mapping made by the residual connection. We replace it with a composition, $\mathcal{I}_i = P_{W_{0}} P_{\text{in}}^\top$, consisting of two permutations --- one to address \textit{issue I} and another for \textit{issue II} --- as follows:
\begin{equation}
\mathbf{z}_i = P_{W_0} \mathbf{z}_{\text{attn}} + \mathcal{I}_i P_{\text{in}} \mathbf{x} = P_{W_0} \mathbf{z}_{\text{attn}} + P_{W_0}\mathbf{x},
\end{equation}
which highlights how the two addends now share the same permutation. An analogous process applies to the second residual connection yielding $\mathbf{z}_{\text{out}}$ (see~\cref{proof:residual} for the full procedure). As a final technical note, we remark that the permutation matrix $P_{W_2}$ associated to the second residual block in \cref{eq:resblock} has to be considered as input permutation for the subsequent layer.
\subsection{Transporting Task Vectors from $\theta_A$ to $\theta_B$}
\label{sec:transportation}
By applying $\pi$ to model $\theta_A$, we would have a functionally equivalent model $\theta_A' = \pi(\theta_A)$ with stronger linear-mode connectivity with $\theta_B$ compared to the original $\theta_A$. However, to allow knowledge transfer from the fine-tuned model $\theta_A^{ft} = \theta_{A} + \tau$ to $\theta_{B}$, we do not apply the permutations directly on $\theta_{A}$, but rather on the task vector $\tau$, as follows:
\begin{eqnarray}
\label{eq:transport}
\operatorname{task \ vector}:& \tau = \theta_A^{ft} - \theta_A,  \\
\operatorname{transport}:& \Tilde{\theta}_B^{ft} = \theta_{B} + \alpha \pi(\tau), 
\end{eqnarray} 
where $\alpha$ is a non-negative scaling factor~\cite{wortsman2022robust} modulating the influence of $\pi(\tau)$ on $\theta_{B}$.

By leveraging the concept of transporting task vectors, we have several notable advantages, especially in a scenario with multiple models fine-tuned on distinct tasks from the same base model $\theta_A$. In this scenario, the weight matching process between $\theta_A$ and $\theta_B$ needs to be conducted only \textbf{once}. Indeed, a permutation set $\pi$ can be established and reused to transfer any number of task vectors. This approach avoids the additional computational costs associated with learning separate transport functions for each transfer. Moreover, transporting multiple task vectors using the same reference model $\theta_A$ allows their combination at destination $\theta_B$, which basically means we could still apply model merging~\cite{wortsman2022model} after re-basin. %

\subsection{Complexity Analysis}
\label{sec:complexity}
In this subsection, we assess the computational complexity of the proposed weight matching procedure. The key insight is that the method is highly efficient compared to full re-training, and scales polynomially with model size.

\begin{proposition}
\label{prop:complexity}
Let $L$ be the number of layers and $d_m$ the embedding dimension of each transformer block. The overall computational complexity of our weight matching procedure is dominated by $O(L \, d_m^3)$. This complexity matches that of Git Re-Basin, making our approach comparably efficient in terms of computational cost.
\end{proposition}

The proof is provided in Appendix~\ref{app:complexity} and illustrates the per-layer contribution of both MLP and attention blocks.

\section{Experiments}

This section is structured into three main parts. Initially, we empirically assess the transportation of task vectors, involving extensive experiments across both visual and natural language processing (NLP) tasks (\cref{sec:taskvectortransport}). Subsequently, we examine the capability of our methodology to align the weights of two Transformer models while maintaining functional equivalence (\cref{sec:weightalignment}). Finally, several ablative studies show the impact of our techniques on addressing multi-head attention layers and residual connections (\cref{sec:ablation}).
\subsection{\methname{} of Task Vectors}
\label{sec:taskvectortransport}

\tit{Visual Classification Tasks.} As reference architecture, we consider the CLIP ViT-B/16 Vision Transformer~\cite{radford2021learning} from Open-CLIP~\cite{cherti2023reproducible}. We refer to $\theta_A$ as the original pre-training weights and $\theta_B$ as those used for the re-basin. We use CommonPool pre-training for $\theta_A$ and Datacomp for $\theta_B$, both cited in~\cite{gadre2024datacomp}. 

Considering the base model $\theta_A$, we fine-tune the corresponding model on several computer vision tasks~\cite{radford2021learning, ilharcoediting}. We employ DTD~\cite{cimpoi2014describing}, EuroSAT~\cite{helber2019eurosat}, GTSRB~\cite{stallkamp2011german}, and SVHN~\cite{netzer2011reading} and obtain multiple, independent fine-tuned models like $\theta_A^{ft} = \theta_A + \tau$. Afterwards, we empirically assess the transportation of $\tau$ to the new weights $\theta_B$. In this respect, we adopt two metrics to characterize the quality of the transported model $\theta_B + \pi(\tau)$: \textit{i)} the zero-shot performance on the original task (\textit{specialized knowledge}), and \textit{ii)} the zero-shot performance on a support, unseen set to evaluate the preservation of \textit{broader capabilities}. In our experiments, ImageNet-R~\cite{hendrycks2021many} serves as a support dataset.

We report the results in~\cref{tab:gain_visual} as drops (-) or gains (+) in accuracy relative to the zero-shot performance of \(\theta_B\). As baselines, we provide the results of \textit{vanilla transportation} (no permutations applied on $\tau$) and those of Git Re-Basin~\cite{ainsworth2023git} and Optimal Transport (OT)~\cite{imfeld2024transformer}, two existing methods for model re-basin. Specifically, the comparison with OT is noteworthy since this approach is designed for Transformer models (like ours). 

As can be seen, our method enhances zero-shot performance on the downstream tasks and preserves generalization on the support dataset, outperforming existing permutation-based methods. Considering the results of our approach, it is particularly noteworthy that we enhance performance on the downstream task while maintaining generalization, all achieved without the use of any data.

In the experiments shown in~\cref{tab:gain_visual}, we consistently  set the scaling coefficient for the (permuted) task vector as $\alpha=1$ (see \cref{eq:transport}). To investigate sensitivity and performance changes while varying $\alpha$, we kindly refer the reader to~\cref{fig:alpha} (more datasets are in~\cref{sec:app_other_datasets}). This illustrates the drop/gain in accuracy for $\theta_B + \alpha \tau$ (\textcolor{blue}{blue}) and our $\theta_B + \alpha \pi(\tau)$ (\textcolor{red}{red}). This drop/gain is measured w.r.t.\ the zero-shot accuracy of $\theta_B$, and \(\alpha\) varies within the range $[0.01, 2.0]$. The outcome is that applying the permuted $\pi(\tau)$ to $\theta_B$ leads to tangible improvements in the downstream task (top row), especially $\alpha \approx 1$. Moreover, when $\alpha \geq 0.5$, the permuted task vector is considerably more reliable in terms of generalization (higher accuracy on the support set).

\begin{figure}[t]
\vskip 0.2in
\begin{center}
\centering\includegraphics[width=.45\textwidth]{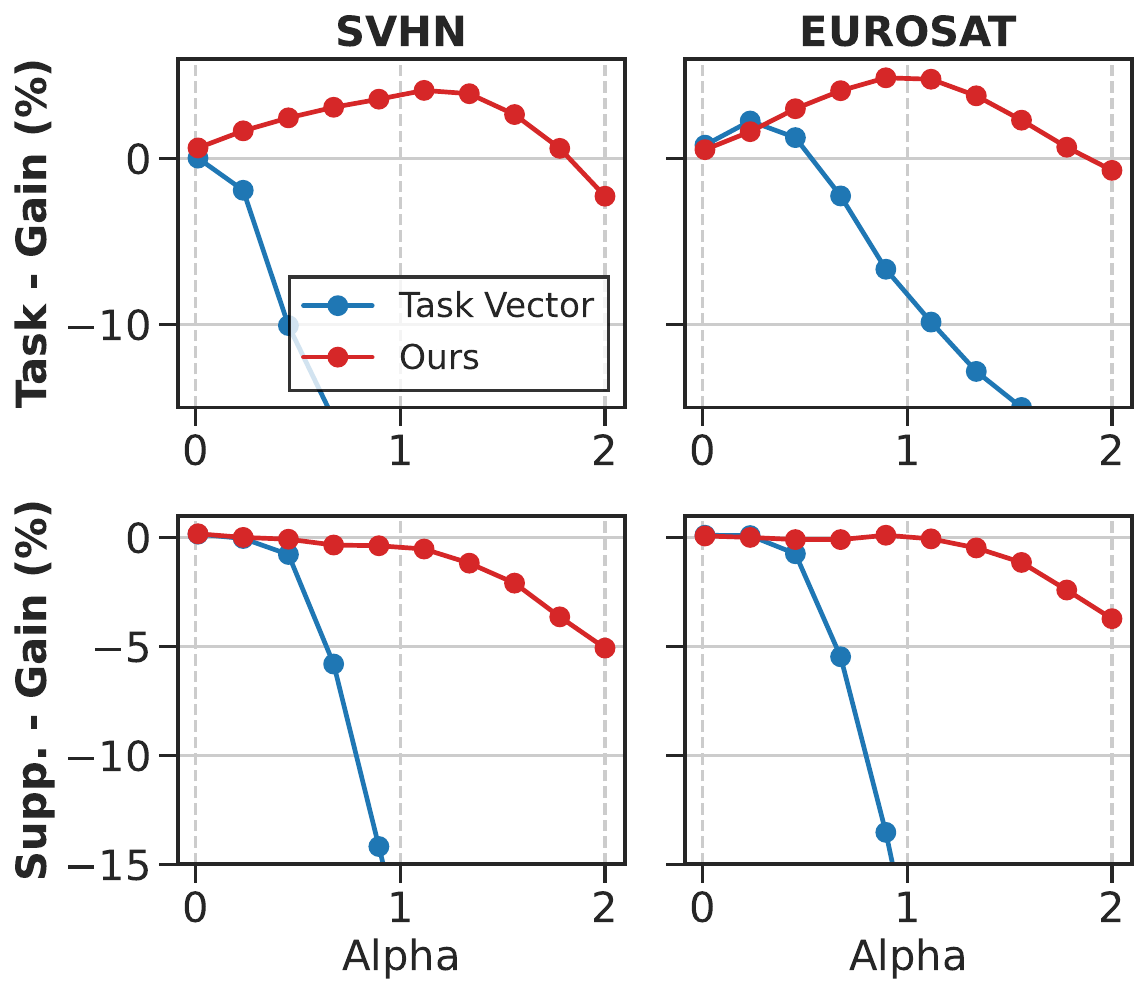}
\caption{Zero-shot gain/drop relative to $\theta_B$ of naive $\theta_B+\alpha\tau$ (\textcolor{blue}{blue}) and our strategy $\theta_B+\alpha\pi(\tau)$ (\textcolor{red}{red}) varying $\alpha$.}
\label{fig:alpha}
\end{center}
\vskip -0.2in
\end{figure}

\begin{table}[t]
    \caption{Comparison of permutation-based methods on NLP tasks, in terms of task accuracy $[\uparrow]$.}
    \label{tab:gain_pokemon}
    \sisetup{
        table-format=2.2(2),
        separate-uncertainty=false,
        table-align-uncertainty=true,
        detect-weight,
    }
    \setlength{\tabcolsep}{3pt} %
    \begin{center}
    \begin{small}
    \begin{sc}
    \begin{tabular}{lgcgc}
        \toprule
        \textit{Method}                                        & \textbf{QQP} & \textbf{SST2} & \textbf{RTE} & \textbf{CoLA} \\
        \midrule
        $\theta_{B}$                                           & $55.00$        & $50.69$         & $54.51$        & $40.94$         \\
        $\theta_{B}$ + $\mathcal{\tau}$                        & -$8.29$        & +$0.23$         & -$2.53$        & -$0.77$         \\
        $\theta_{B}$ + $\mathcal{\tau}$ \textit{(OT)}          & -$8.31$        & +$5.39$         & -$1.08$        & -$1.25$         \\
        $\theta_{B}$ + $\mathcal{\tau}$ \textit{(GiT Re-Basin)} & +$3.58$        & +$5.73$         & +$2.17$        & +$1.44$         \\
        \midrule
        \textbf{TransFusion (Ours)}                            & +$\mathbf{6.50}$ & +$\mathbf{5.96}$  & +$\mathbf{3.61}$ & +$\mathbf{2.49}$  \\
        \bottomrule
    \end{tabular}
    \end{sc}
    \end{small}
    \end{center}
\end{table}

\tit{NLP Classification Tasks.} Herein, we investigate a different setting that involves closed-vocabulary text classification --- specifically, a set of tasks from the GLUE benchmark~\cite{wang2018glue}. We consider a model $\theta=\{\phi, \omega\}$ composed of a pre-trained Transformer encoder $\phi$ and a classification head $\omega$. We then evaluate the transport of the learned task vector $\tau_\phi=\phi^{ft}_A-\phi_A$ on a new feature extractor $\phi_B$. As access to data of the downstream task is restricted, we are unable to train a new classifier for $\theta_B$: consequently, we re-use the originally fine-tuned classifier, denoted as $\omega^{ft}$. The goal is to evaluate whether transporting the task vector $\tau_\phi$ aligns the representation yielded by $\phi_B + \pi(\tau_\phi)$ with the original, fine-tuned classifier $\omega^{ft}$. In our experiments, we employed two variants of the ViT-B-16 text encoder, pretrained respectively on the \texttt{commonpool-l-s1b-b8k} ($\theta_A$) and \texttt{datacomp-l-s1b-b8k} ($\theta_B$)~\cite{gadre2024datacomp}.

\cref{tab:gain_pokemon} presents the evaluation for the GLUE benchmark. Unexpectedly, simply applying the classification head from the original feature extractor $\phi_A$ yields poor performance (see first line of~\cref{tab:gain_pokemon}, $\theta_B$). %
On the other hand, transporting $\tau_\phi$ with Git Re-Basin and Optimal Transport performs reasonably, with good gains on QQP and SST2. Moreover, our approach leads to the highest and more consistent performance gains, highlighting the potential of our framework.
\subsection{\methname{} Improves Alignment and Preserves Functional Equivalence}
\label{sec:weightalignment}
\begin{figure}[t]
\vskip 0.2in
\centering
\includegraphics[width=0.46\textwidth]{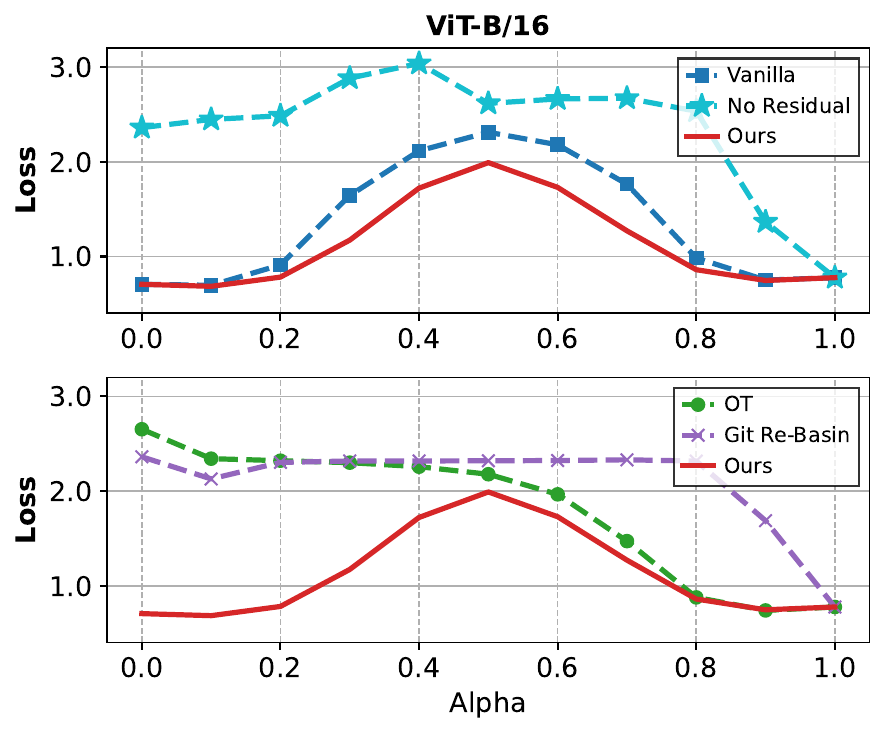}
\caption{Loss values on CIFAR-10 test set during model interpolation. Top: Our permutation approach vs. vanilla interpolation and no residual variant. Bottom: Comparison with Optimal Transport and Git Re-Basin methods, which fail to preserve functional equivalence as $\alpha \rightarrow 0$.}
\vskip -0.2in
\label{fig:barrier}
\end{figure}

While the previous analyses focus on transferring task vectors, we now delve into the effectiveness of our approach in terms of weight alignment. In detail, we consider two ViT-B/16 models~\cite{dosovitskiy2020image} $A$ and $B$ trained independently on CIFAR-10~\cite{krizhevsky2009learning} from scratch, which means they underwent different initializations and batch orders. After training, we apply our permutation strategy to $\theta_A$ and analyze the resulting alignment of $\pi(\theta_A)$ and $\theta_B$ in terms of \textit{linear mode connectivity}: following~\cite{ainsworth2023git}, we evaluate the loss landscape $\mathcal{L}\left((1-\alpha) \pi\left(\theta_A\right) + \alpha \theta_B\right), \alpha \in [0,1]$ while interpolating between the two models. \\
As shown by~\cref{fig:barrier} (top), applying the permutation $\pi$ yields an interpolated model that exhibits consistent lower loss compared to the vanilla approach that does not permute $\theta_A$ ($\pi(\theta_A) = \theta_A$). Moreover, the \textit{no residual} approach underscores the critical role of properly handling residual connections in our method --- both during the permutation of $\theta_A$ and throughout the interpolation process between models.
Similarly, in~\cref{fig:barrier} (bottom) we assess the loss landscape using $\pi$ derived from the Optimal Transport (OT) hard alignment method~\cite{imfeld2024transformer} and Git Re-Basin~\cite{ainsworth2023git}. While these works show a more favorable loss landscape than naive interpolation, we observe that the resulting interpolated model struggles and achieves high loss when $\alpha \to 0$, highlighting that both OT and Git Re-Basin do not preserve functional equivalence. We conjecture that such a lack stems from potential weaknesses in effectively permuting layers featuring residual connections and multi-head attention blocks. In virtue of the results achieved by our approach, we can claim it represents the first successful data-free method to interpolate between two Transformer models in weight space, while ensuring the functional equivalence of $\pi(\theta_A)$.
\subsection{Ablative Analysis}
\label{sec:ablation}
\tit{On the Strategy to Manage Multiple Heads.} We herein explore the significance of an appropriate policy for permuting the projection layers within multi-head attention mechanisms. Specifically, we present the outcomes of transferring $\tau$ with varying strategies to permute attention projection layers. As potential alternatives, we firstly consider \textbf{brute force alignment}, which considers all possible head pair combinations within each attention layer. Then, for each candidate pair, the \textit{intra-head} alignment cost is computed by optimizing the objective in \cref{eq:intra_head_obj}. The final permutation is then derived with the Hungarian algorithm, which selects pairs with the highest intra-alignment scores.

\begin{table}[t]
    \caption{TransFusion results by head alignment strategy.}
    \label{tab:head_align}
    \setlength{\tabcolsep}{3pt}
    \sisetup{
        table-format=2.2,
        separate-uncertainty=false,
        table-align-uncertainty=true,
        detect-weight,
    }
    \begin{center}
    \begin{small}
    \begin{sc}
    \setlength{\tabcolsep}{2pt} %
    \begin{tabular}{l*{3}{>{}g>{}c}}
        \toprule
        \multirow{2}{*}{\textit{Head Align.}} 
        & \multicolumn{2}{c}{\textbf{EuroSAT}} 
        & \multicolumn{2}{c}{\textbf{GTSRB}} 
        & \multicolumn{2}{c}{\textbf{SVHN}} \\
        & Task & Supp. & Task & Supp. & Task & Supp. \\
        \midrule
        $\theta_{B}$ \textit{zero-shot}       
        & $49.02$ & $68.73$ 
        & $43.42$ & $68.73$ 
        & $45.97$ & $68.73$ \\
        
        \textit{Brute Force}                 
        & +$1.32$ & -$0.21$ 
        & +$0.60$ & -$0.46$ 
        & +$3.39$ & -$0.40$ \\
        
        \textit{No Att-Align}                
        & +$2.22$ & -$0.47$ 
        & +$0.71$ & +$\mathbf{0.05}$ 
        & +$0.24$ & -$\mathbf{0.08}$ \\
        
        \midrule
        \textbf{Ours (Full)}                 
        & +$\mathbf{4.95}$ & -$\mathbf{0.06}$ 
        & +$\mathbf{1.10}$ & -$0.40$ 
        & +$\mathbf{3.64}$ & -$0.48$ \\
        \bottomrule
    \end{tabular}
    \end{sc}
    \end{small}
    \end{center}

\end{table}

After, we compare our approach with one that pairs heads in  $A$ and $B$ according to their natural order, thereby avoiding \textit{head contamination} by design. Nevertheless, the preservation of original head ordering comes at the cost of ignoring functional mismatches between attention units. We refer to this further baseline as \textbf{no attention alignment}.

The results of these experiments are detailed in \cref{tab:head_align}. Our attention-alignment strategy achieves superior performance on the downstream task ($\theta_B + \pi(\tau)$) compared to alternative approaches, while maintaining comparable zero-shot capabilities. The comparison with the brute force approach underscores the effectiveness of our permutation-invariant costs in modeling inter-head relationships, demonstrating superior performance over a brute force alignment that optimizes for the best match within each candidate pair of heads. Furthermore, our results suggest that preserving the original head ordering (as in the no-alignment strategy) yields better performance than brute-force inter-head matching for two of the three experimental tasks. This underscores the menace of head contamination and the importance of structure-aware alignment methodologies over indiscriminate similarity maximization.
\begin{table}[t]
    \caption{Fine-tuning results with permuted task vectors.}
    \label{tab:fewshot}
    \begin{center}
    \begin{small}
    \begin{sc}
    \begin{tabular}{lgcgc}
        \toprule
        \textit{Method} & \textbf{EuroSAT} & \textbf{DTD} & \textbf{GTSRB} & \textbf{SVHN} \\
        \midrule
        $\theta_{B}$ \textit{zero-shot}  & $49.02$ & $47.50$ & $43.42$ & $45.97$ \\
        $\theta_{B}+\alpha\tau$ & +$7.93$ & -$1.44$ & +$4.70$ & -$15.98$ \\
        $\theta_{B}+\alpha\pi(\tau)$ & +$\mathbf{10.00}$ & +$\mathbf{1.21}$ & +$\mathbf{6.80}$ & +$\mathbf{10.52}$ \\
        \bottomrule
    \end{tabular}
    \end{sc}
    \end{small}
    \end{center}
    \vskip -0.1in
\end{table}

\tit{Few Shot Fine-tuning.}There are practical scenarios in which retaining data is infeasible. If such constraints are not present, our proposed method can be effectively combined with fine-tuning. To illustrate this, we follow \cite{zhang2024knowledge} and start with a small subset consisting of $10$ shots per class, learning a scaling coefficient per layer, denoted as $\alpha = [\alpha_1, \dots, \alpha_{|L|}]$. The results in \cref{tab:fewshot} clearly indicate a substantial improvement when fine-tuning a model that has undergone re-basin using our approach, represented as $\theta_{B}+\alpha\pi(\tau)$. In contrast, fine-tuning directly from $\theta_{B}+\alpha\tau$, without permutation, yields inferior outcomes. This emphasizes that re-basing and fine-tuning should not be considered mutually exclusive but complementary strategies. 

\section{Related Work}
\label{sec:related}
\textbf{\textit{Mode Connectivity}} occurs when paths of nearly constant loss connect different solutions within the loss landscape of neural networks (NNs)~\cite{garipov2018loss, freeman2017topology, garipov2018loss, draxler2018essentially}. When such paths are linear, we refer to \textit{linear} mode connectivity (LMC)~\cite{frankle2020linear}.~\citet{entezari2022the} conjecture that solutions found by Stochastic Gradient Descent (SGD) can be linearly connected when accounting for permutation symmetries. Motivated by this, several works first align the models into a shared optimization space by permuting their neurons, and then merge them through a simple average~\citep{ainsworth2023git, jordan2023repair, stoica2024zipit, pena2023re, crisostomi2024c, navon2023equivariant, singh2020model}. Most relevant to our work,~\citet{imfeld2024transformer} applies optimal transport to align activations in transformer-based networks. However, unlike the latter, our method accounts for the logical division of multi-head attention projections, preserving the functional equivalence of the permuted models.

\textbf{\textit{Weight Interpolation and Task Arithmetic.}}
Emerging research reveals that the output of NNs can be manipulated through algebraic operations in weight space \cite{ilharco2022patching, wortsman2022model}. Central to this paradigm are task vectors $\tau$~\cite{ilharcoediting}, which encode task-specific knowledge and exhibit compositional properties: combining vectors through addition enables multi-task generalization while their negation can suppress learned behaviors without significantly impacting unrelated tasks.

Beyond arithmetic, weight interpolations further unlock unexpected capabilities: blending fine-tuned and pre-trained weights often yields single-task performance superior to standalone fine-tuning~\cite{frankle2020linear, izmailov2018averaging, matena2022merging,rame2022recycling,rame2022diverse,wortsman2022robust}, suggesting a reconciliation of specialized adaptation with generalization capabilities. Multi-task merging via parameter averaging~\cite{ilharco2022patching,ilharcoediting, wortsman2022model, yadav2024ties} not only circumvents catastrophic forgetting~\cite{french1999catastrophic, mccloskey1989catastrophic, porrello2024second} but synthesizes models that retain diverse expertise, even serving as superior starting points for future adaptation~\cite{choshen2022fusingfinetunedmodelsbetter}. 
The benefits of weight ensembles and interpolations extend beyond just fine-tuned models; they also apply to models that are trained from scratch. Techniques such as those proposed by~\cite{ainsworth2023git, singh2020model}, leverage permutation symmetries to facilitate coherent interpolation between models trained in different ways. Collectively, these findings position weight-space manipulation as a scalable toolkit for resource-efficient model engineering, where arithmetic and interpolation replace brute-force retraining.

\section{Discussion and Conclusions}
For \methname\ to succeed, the source expert must perform strongly on the target task.
This insight directly explains the diminished results on DTD in~\cref{tab:gain_visual}: while our fine-tuning achieves performance well above $95\%$ on all other datasets, DTD reaches only $75\%$ accuracy. We attribute this degraded task vector quality to DTD's inherently challenging characteristics --- featuring merely 40 examples per class on average, compared to roughly 1,000 examples per class in our other datasets.
Although our proposal delivers substantial gains without requiring access to data from the original task, we acknowledge significant room for improvement. We believe most of the lost performance could be recovered through two complementary directions: incorporating a modest amount of additional data, or employing more sophisticated matching metrics for permutation discovery that better reflect activation distributions rather than relying on simple dot product similarity for linear layers. Such extensions represent promising avenues for future investigation.

\section*{Acknowledgments}
This work was supported by the Italian Ministerial grants PRIN 2022: "B-Fair: Bias-Free Artificial Intelligence Methods for Automated Visual Recognition" (CUP E53D23008010006) and by the University of Modena and Reggio Emilia and Fondazione di Modena through the "Fondo di Ateneo per la Ricerca - FAR 2024" (CUP E93C24002080007). The work also received funding from DECIDER, the European Union’s Horizon 2020 research and innovation programme under GA No. 965193 and “AIDA: explAinable multImodal
Deep learning for personAlized oncology” (Project Code
20228MZFAA). We acknowledge the CINECA award under the ISCRA initiative for providing high-performance computing resources.

\section*{Impact Statement}

Our approach to transferring fine-tuning between model versions without requiring re-training or additional data holds the potential to significantly lower the barriers to maintaining state-of-the-art AI technologies. By enabling seamless updates, this method could empower organizations, especially those with limited resources, to keep pace with technological advancements, thereby democratizing access to sophisticated AI tools. This democratization can foster innovation across diverse sectors, potentially leveling the playing field between large entities and smaller, resource-constrained organizations.

The idea of aligning models with minor variations in architecture, such as different layer counts, is also worth exploring. One simple approach could involve selectively pruning layers from the model with more layers to match its counterpart—removing redundant or unimportant layers—while an alternative strategy might replicate the last block of the smaller network multiple times to achieve parity. We plan to investigate these strategies further in future work.

However, as we facilitate easier transitions between model versions, it is crucial to ensure that these updates do not compromise the integrity of the models. Dependence on outdated or poorly validated models poses significant risks, particularly when used in critical applications. Therefore, it is imperative that as this technology is adopted, continuous efforts are made to monitor, validate, and refine these models to safeguard against biases and errors that may arise from rapid model evolution. Future research should focus on developing robust frameworks for evaluating the ethical implications, fairness, and transparency of AI models as they evolve, ensuring that advancements in AI technology are implemented responsibly and ethically.

\bibliography{main}

\begin{thebibliography}{44}
\providecommand{\natexlab}[1]{#1}
\providecommand{\url}[1]{\texttt{#1}}
\expandafter\ifx\csname urlstyle\endcsname\relax
  \providecommand{\doi}[1]{doi: #1}\else
  \providecommand{\doi}{doi: \begingroup \urlstyle{rm}\Url}\fi

\bibitem[Ainsworth et~al.(2023)Ainsworth, Hayase, and Srinivasa]{ainsworth2023git}
Ainsworth, S., Hayase, J., and Srinivasa, S.
\newblock {Git Re-Basin: Merging Models modulo Permutation Symmetries}.
\newblock In \emph{International Conference on Learning Representations}, 2023.

\bibitem[Cherti et~al.(2023)Cherti, Beaumont, Wightman, Wortsman, Ilharco, Gordon, Schuhmann, Schmidt, and Jitsev]{cherti2023reproducible}
Cherti, M., Beaumont, R., Wightman, R., Wortsman, M., Ilharco, G., Gordon, C., Schuhmann, C., Schmidt, L., and Jitsev, J.
\newblock {Reproducible scaling laws for contrastive language-image learning}.
\newblock In \emph{Proceedings of the IEEE conference on Computer Vision and Pattern Recognition}, 2023.

\bibitem[Choshen et~al.(2022)Choshen, Venezian, Slonim, and Katz]{choshen2022fusingfinetunedmodelsbetter}
Choshen, L., Venezian, E., Slonim, N., and Katz, Y.
\newblock Fusing finetuned models for better pretraining.
\newblock \emph{arXiv preprint arXiv:2204.03044}, 2022.

\bibitem[Cimpoi et~al.(2014)Cimpoi, Maji, Kokkinos, Mohamed, and Vedaldi]{cimpoi2014describing}
Cimpoi, M., Maji, S., Kokkinos, I., Mohamed, S., and Vedaldi, A.
\newblock {Describing Textures in the Wild}.
\newblock In \emph{Proceedings of the IEEE conference on Computer Vision and Pattern Recognition}, 2014.

\bibitem[Crisostomi et~al.(2024)Crisostomi, Fumero, Baieri, Bernard, and Rodola]{crisostomi2024c}
Crisostomi, D., Fumero, M., Baieri, D., Bernard, F., and Rodola, E.
\newblock {C\textsuperscript{2}M\textsuperscript{3}: Cycle-Consistent Multi-Model Merging}.
\newblock \emph{Advances in Neural Information Processing Systems}, 2024.

\bibitem[Dosovitskiy et~al.(2021)Dosovitskiy, Beyer, Kolesnikov, Weissenborn, Zhai, Unterthiner, Dehghani, Minderer, Heigold, Gelly, Uszkoreit, and Houlsby]{dosovitskiy2020image}
Dosovitskiy, A., Beyer, L., Kolesnikov, A., Weissenborn, D., Zhai, X., Unterthiner, T., Dehghani, M., Minderer, M., Heigold, G., Gelly, S., Uszkoreit, J., and Houlsby, N.
\newblock {An Image is Worth 16x16 Words: Transformers for Image Recognition at Scale}.
\newblock In \emph{International Conference on Learning Representations}, 2021.

\bibitem[Draxler et~al.(2018)Draxler, Veschgini, Salmhofer, and Hamprecht]{draxler2018essentially}
Draxler, F., Veschgini, K., Salmhofer, M., and Hamprecht, F.
\newblock {Essentially No Barriers in Neural Network Energy Landscape}.
\newblock In \emph{International Conference on Machine Learning}, 2018.

\bibitem[Entezari et~al.(2022)Entezari, Sedghi, Saukh, and Neyshabur]{entezari2022the}
Entezari, R., Sedghi, H., Saukh, O., and Neyshabur, B.
\newblock {The Role of Permutation Invariance in Linear Mode Connectivity of Neural Networks}.
\newblock In \emph{International Conference on Learning Representations}, 2022.

\bibitem[Frankle et~al.(2020)Frankle, Dziugaite, Roy, and Carbin]{frankle2020linear}
Frankle, J., Dziugaite, G.~K., Roy, D., and Carbin, M.
\newblock {Linear Mode Connectivity and the Lottery Ticket Hypothesis}.
\newblock In \emph{International Conference on Machine Learning}, 2020.

\bibitem[Freeman \& Bruna(2017)Freeman and Bruna]{freeman2017topology}
Freeman, C.~D. and Bruna, J.
\newblock {Topology and Geometry of Half-Rectified Network Optimization}.
\newblock In \emph{International Conference on Learning Representations}, 2017.

\bibitem[French(1999)]{french1999catastrophic}
French, R.~M.
\newblock {Catastrophic forgetting in connectionist networks}.
\newblock \emph{Trends in Cognitive Sciences}, 3\penalty0 (4), 1999.

\bibitem[Gadre et~al.(2024)Gadre, Ilharco, Fang, Hayase, Smyrnis, Nguyen, Marten, Wortsman, Ghosh, Zhang, et~al.]{gadre2024datacomp}
Gadre, S.~Y., Ilharco, G., Fang, A., Hayase, J., Smyrnis, G., Nguyen, T., Marten, R., Wortsman, M., Ghosh, D., Zhang, J., et~al.
\newblock {Datacomp: In search of the next generation of multimodal datasets}.
\newblock \emph{Advances in Neural Information Processing Systems}, 2024.

\bibitem[Garipov et~al.(2018)Garipov, Izmailov, Podoprikhin, Vetrov, and Wilson]{garipov2018loss}
Garipov, T., Izmailov, P., Podoprikhin, D., Vetrov, D.~P., and Wilson, A.~G.
\newblock {Loss Surfaces, Mode Connectivity, and Fast Ensembling of DNNs}.
\newblock \emph{Advances in Neural Information Processing Systems}, 2018.

\bibitem[Helber et~al.(2019)Helber, Bischke, Dengel, and Borth]{helber2019eurosat}
Helber, P., Bischke, B., Dengel, A., and Borth, D.
\newblock {EuroSAT: A Novel Dataset and Deep Learning Benchmark for Land Use and Land Cover Classification}.
\newblock \emph{IEEE Journal of Selected Topics in Applied Earth Observations and Remote Sensing}, 12\penalty0 (7), 2019.

\bibitem[Hendrycks et~al.(2021)Hendrycks, Basart, Mu, Kadavath, Wang, Dorundo, Desai, Zhu, Parajuli, Guo, Song, Steinhardt, and Gilmer]{hendrycks2021many}
Hendrycks, D., Basart, S., Mu, N., Kadavath, S., Wang, F., Dorundo, E., Desai, R., Zhu, T., Parajuli, S., Guo, M., Song, D., Steinhardt, J., and Gilmer, J.
\newblock {The Many Faces of Robustness: A Critical Analysis of Out-of-Distribution Generalization}.
\newblock \emph{IEEE International Conference on Computer Vision}, 2021.

\bibitem[Ilharco et~al.(2021)Ilharco, Wortsman, Wightman, Gordon, Carlini, Taori, Dave, Shankar, Namkoong, Miller, Hajishirzi, Farhadi, and Schmidt]{ilharco_gabriel_2021_5143773}
Ilharco, G., Wortsman, M., Wightman, R., Gordon, C., Carlini, N., Taori, R., Dave, A., Shankar, V., Namkoong, H., Miller, J., Hajishirzi, H., Farhadi, A., and Schmidt, L.
\newblock {OpenCLIP}, 2021.

\bibitem[Ilharco et~al.(2022)Ilharco, Wortsman, Gadre, Song, Hajishirzi, Kornblith, Farhadi, and Schmidt]{ilharco2022patching}
Ilharco, G., Wortsman, M., Gadre, S.~Y., Song, S., Hajishirzi, H., Kornblith, S., Farhadi, A., and Schmidt, L.
\newblock {Patching open-vocabulary models by interpolating weights}.
\newblock \emph{Advances in Neural Information Processing Systems}, 2022.

\bibitem[Ilharco et~al.(2023)Ilharco, Ribeiro, Wortsman, Schmidt, Hajishirzi, and Farhadi]{ilharcoediting}
Ilharco, G., Ribeiro, M.~T., Wortsman, M., Schmidt, L., Hajishirzi, H., and Farhadi, A.
\newblock {Editing Models with Task Arithmetic}.
\newblock In \emph{International Conference on Learning Representations}, 2023.

\bibitem[Imfeld et~al.(2024)Imfeld, Graldi, Giordano, Hofmann, Anagnostidis, and Singh]{imfeld2024transformer}
Imfeld, M., Graldi, J., Giordano, M., Hofmann, T., Anagnostidis, S., and Singh, S.~P.
\newblock {Transformer Fusion with Optimal Transport}.
\newblock In \emph{International Conference on Learning Representations}, 2024.

\bibitem[Izmailov et~al.(2018)Izmailov, Podoprikhin, Garipov, Vetrov, and Wilson]{izmailov2018averaging}
Izmailov, P., Podoprikhin, D., Garipov, T., Vetrov, D., and Wilson, A.~G.
\newblock {Averaging Weights Leads to Wider Optima and Better Generalization}.
\newblock \emph{arXiv preprint arXiv:1803.05407}, 2018.

\bibitem[Jonker \& Volgenant(1988)Jonker and Volgenant]{jonker1988shortest}
Jonker, R. and Volgenant, T.
\newblock {A shortest augmenting path algorithm for dense and sparse linear assignment problems}.
\newblock In \emph{Papers of the 16th Annual Meeting of DGOR in Cooperation with NSOR/Vortr{\"a}ge der 16. Jahrestagung der DGOR zusammen mit der NSOR}, 1988.

\bibitem[Jordan et~al.(2023)Jordan, Sedghi, Saukh, Entezari, and Neyshabur]{jordan2023repair}
Jordan, K., Sedghi, H., Saukh, O., Entezari, R., and Neyshabur, B.
\newblock {REPAIR: REnormalizing Permuted Activations for Interpolation Repair}.
\newblock In \emph{International Conference on Learning Representations}, 2023.

\bibitem[Jovanovi{\'c} \& Stani{\'c}(2012)Jovanovi{\'c} and Stani{\'c}]{jovanovic2012spectral}
Jovanovi{\'c}, I. and Stani{\'c}, Z.
\newblock {Spectral distances of graphs}.
\newblock \emph{Linear Algebra and its Applications}, 436\penalty0 (5), 2012.

\bibitem[Krizhevsky et~al.(2009)Krizhevsky, Hinton, et~al.]{krizhevsky2009learning}
Krizhevsky, A., Hinton, G., et~al.
\newblock {Learning Multiple Layers of Features from Tiny Images}.
\newblock \emph{Technical Report, University of Toronto}, 2009.

\bibitem[Lu et~al.(2024)Lu, Chen, Williamson, Chen, Liang, Ding, Jaume, Odintsov, Le, Gerber, et~al.]{lu2024avisionlanguage}
Lu, M.~Y., Chen, B., Williamson, D.~F., Chen, R.~J., Liang, I., Ding, T., Jaume, G., Odintsov, I., Le, L.~P., Gerber, G., et~al.
\newblock {A visual-language foundation model for computational pathology}.
\newblock \emph{Nature Medicine}, 30, 2024.

\bibitem[Mall et~al.(2024)Mall, Phoo, Liu, Vondrick, Hariharan, and Bala]{mall2024remote}
Mall, U., Phoo, C.~P., Liu, M.~K., Vondrick, C., Hariharan, B., and Bala, K.
\newblock {Remote Sensing Vision-Language Foundation Models without Annotations via Ground Remote Alignment}.
\newblock In \emph{International Conference on Learning Representations}, 2024.

\bibitem[Matena \& Raffel(2022)Matena and Raffel]{matena2022merging}
Matena, M.~S. and Raffel, C.~A.
\newblock {Merging Models with Fisher-Weighted Averaging}.
\newblock \emph{Advances in Neural Information Processing Systems}, 35, 2022.

\bibitem[McCloskey \& Cohen(1989)McCloskey and Cohen]{mccloskey1989catastrophic}
McCloskey, M. and Cohen, N.~J.
\newblock {Catastrophic Interference in Connectionist Networks: The Sequential Learning Problem}.
\newblock In \emph{Psychology of learning and motivation}, volume~24, pp.\  109--165. Academic Press, 1989.

\bibitem[Navon et~al.(2023)Navon, Shamsian, Achituve, Fetaya, Chechik, and Maron]{navon2023equivariant}
Navon, A., Shamsian, A., Achituve, I., Fetaya, E., Chechik, G., and Maron, H.
\newblock {Equivariant Architectures for Learning in Deep Weight Spaces}.
\newblock In \emph{International Conference on Machine Learning}, 2023.

\bibitem[Netzer et~al.(2011)Netzer, Wang, Coates, Bissacco, Wu, Ng, et~al.]{netzer2011reading}
Netzer, Y., Wang, T., Coates, A., Bissacco, A., Wu, B., Ng, A.~Y., et~al.
\newblock {Reading Digits in Natural Images with Unsupervised Feature Learning}.
\newblock In \emph{Neural Information Processing Systems Workshops}. Granada, 2011.

\bibitem[Ortiz-Jimenez et~al.(2024)Ortiz-Jimenez, Favero, and Frossard]{ortiz2024task}
Ortiz-Jimenez, G., Favero, A., and Frossard, P.
\newblock {Task Arithmetic in the Tangent Space: Improved Editing of Pre-Trained Models}.
\newblock \emph{Advances in Neural Information Processing Systems}, 2024.

\bibitem[Pe{\~n}a et~al.(2023)Pe{\~n}a, Medeiros, Dubail, Aminbeidokhti, Granger, and Pedersoli]{pena2023re}
Pe{\~n}a, F. A.~G., Medeiros, H.~R., Dubail, T., Aminbeidokhti, M., Granger, E., and Pedersoli, M.
\newblock {Re-basin via implicit Sinkhorn differentiation}.
\newblock In \emph{Proceedings of the IEEE conference on Computer Vision and Pattern Recognition}, 2023.

\bibitem[Porrello et~al.(2025)Porrello, Bonicelli, Buzzega, Millunzi, Calderara, and Cucchiara]{porrello2024second}
Porrello, A., Bonicelli, L., Buzzega, P., Millunzi, M., Calderara, S., and Cucchiara, R.
\newblock {A Second-Order Perspective on Model Compositionality and Incremental Learning}.
\newblock In \emph{International Conference on Learning Representations}, 2025.

\bibitem[Radford et~al.(2021)Radford, Kim, Hallacy, Ramesh, Goh, Agarwal, Sastry, Askell, Mishkin, Clark, et~al.]{radford2021learning}
Radford, A., Kim, J.~W., Hallacy, C., Ramesh, A., Goh, G., Agarwal, S., Sastry, G., Askell, A., Mishkin, P., Clark, J., et~al.
\newblock {Learning Transferable Visual Models From Natural Language Supervision}.
\newblock In \emph{International Conference on Machine Learning}, 2021.

\bibitem[Rame et~al.(2022)Rame, Kirchmeyer, Rahier, Rakotomamonjy, patrick gallinari, and Cord]{rame2022diverse}
Rame, A., Kirchmeyer, M., Rahier, T., Rakotomamonjy, A., patrick gallinari, and Cord, M.
\newblock {Diverse Weight Averaging for Out-of-Distribution Generalization}.
\newblock In \emph{Advances in Neural Information Processing Systems}, 2022.

\bibitem[Ram{\'e} et~al.(2023)Ram{\'e}, Ahuja, Zhang, Cord, Bottou, and Lopez-Paz]{rame2022recycling}
Ram{\'e}, A., Ahuja, K., Zhang, J., Cord, M., Bottou, L., and Lopez-Paz, D.
\newblock {Model Ratatouille: Recycling Diverse Models for Out-of-Distribution Generalization}.
\newblock In \emph{International Conference on Machine Learning}, 2023.

\bibitem[Singh \& Jaggi(2020)Singh and Jaggi]{singh2020model}
Singh, S.~P. and Jaggi, M.
\newblock {Model Fusion via Optimal Transport}.
\newblock \emph{Advances in Neural Information Processing Systems}, 2020.

\bibitem[Stallkamp et~al.(2011)Stallkamp, Schlipsing, Salmen, and Igel]{stallkamp2011german}
Stallkamp, J., Schlipsing, M., Salmen, J., and Igel, C.
\newblock {The German Traffic Sign Recognition Benchmark: A multi-class classification competition}.
\newblock In \emph{The 2011 international joint conference on neural networks}. IEEE, 2011.

\bibitem[Stoica et~al.(2024)Stoica, Bolya, Bjorner, Ramesh, Hearn, and Hoffman]{stoica2024zipit}
Stoica, G., Bolya, D., Bjorner, J., Ramesh, P., Hearn, T., and Hoffman, J.
\newblock {ZipIt! Merging Models from Different Tasks without Training}.
\newblock In \emph{International Conference on Learning Representations}, 2024.

\bibitem[Wang et~al.(2019)Wang, Singh, Michael, Hill, Levy, and Bowman]{wang2018glue}
Wang, A., Singh, A., Michael, J., Hill, F., Levy, O., and Bowman, S.~R.
\newblock {GLUE A Multi-Task Benchmark and Analysis Platform for Natural Language Understanding}.
\newblock In \emph{International Conference on Learning Representations}, 2019.

\bibitem[Wortsman et~al.(2022{\natexlab{a}})Wortsman, Ilharco, Gadre, Roelofs, Gontijo-Lopes, Morcos, Namkoong, Farhadi, Carmon, Kornblith, et~al.]{wortsman2022model}
Wortsman, M., Ilharco, G., Gadre, S.~Y., Roelofs, R., Gontijo-Lopes, R., Morcos, A.~S., Namkoong, H., Farhadi, A., Carmon, Y., Kornblith, S., et~al.
\newblock {Model soups: averaging weights of multiple fine-tuned models improves accuracy without increasing inference time}.
\newblock In \emph{International Conference on Machine Learning}, 2022{\natexlab{a}}.

\bibitem[Wortsman et~al.(2022{\natexlab{b}})Wortsman, Ilharco, Kim, Li, Kornblith, Roelofs, Lopes, Hajishirzi, Farhadi, Namkoong, et~al.]{wortsman2022robust}
Wortsman, M., Ilharco, G., Kim, J.~W., Li, M., Kornblith, S., Roelofs, R., Lopes, R.~G., Hajishirzi, H., Farhadi, A., Namkoong, H., et~al.
\newblock {Robust fine-tuning of zero-shot models}.
\newblock In \emph{Proceedings of the IEEE conference on Computer Vision and Pattern Recognition}, 2022{\natexlab{b}}.

\bibitem[Yadav et~al.(2024)Yadav, Tam, Choshen, Raffel, and Bansal]{yadav2024ties}
Yadav, P., Tam, D., Choshen, L., Raffel, C.~A., and Bansal, M.
\newblock {TIES-Merging: Resolving Interference When Merging Models}.
\newblock \emph{Advances in Neural Information Processing Systems}, 36, 2024.

\bibitem[Zhang et~al.(2024)Zhang, Albert, Rodriguez-Opazo, van~den Hengel, and Abbasnejad]{zhang2024knowledge}
Zhang, F.~Z., Albert, P., Rodriguez-Opazo, C., van~den Hengel, A., and Abbasnejad, E.
\newblock {Knowledge Composition using Task Vectors with Learned Anisotropic Scaling}.
\newblock \emph{Advances in Neural Information Processing Systems}, 2024.

\end{thebibliography}
\bibliographystyle{icml2025}

\newpage 

\clearpage
\appendix
\section{Appendix}

\setlength{\belowdisplayskip}{8pt} 
\setlength{\abovedisplayskip}{8pt}

\subsection{On the Invariance to Permutations of our Metric for Inter-head Alignment}
\label{sec:teorema}
\begin{proposition}
    Let $h \in \mathbb{R}^{m \times n}$ be arbitrary. For any $h$, denote its singular values by $\sigma(h)=(\sigma_1(h), \sigma_2(h), \dots,\sigma_{\min(m,n)}(h))$, where $\sigma_1(h)\ge \sigma_2(h)\ge \cdots \ge 0$. For two matrices $h_1, h_2$ of the same shape, define
    \begin{equation}
    d_p(h_1,h_2) = \| \sigma(h_1) - \sigma(h_2)\|_p\,,
    \end{equation}    
    where $\| \cdot \|_p$ is the usual $p$-norm for vectors. Then, for any permutation matrices $P_r \in \mathbb{R}^{m\times m}$ and $P_c \in\mathbb{R}^{n\times n}$, the row- and column-permuted matrix
    \begin{equation}
    h' = P_r h P_c
    \end{equation}    
    has exactly the same singular values as $h$. In particular,
    \begin{equation}
    d_p(h,h') = 0 
    \end{equation}    
    for every $p$, making $d$ invariant under row- and column-permutations of $h$.
\end{proposition}

\begin{proof}
\small
    If $P$ is a permutation matrix, then $P^\top P = I$, \emph{i.e.} it is orthogonal. Furthermore, the singular values of any matrix $h$ are given by the square root of the eigenvalues of  $h^\top h$. If $h' = P_r \, h \, P_c$, then
   \begin{align}
     (h')^\top (h') &=
     (P_r \, h \, P_c)^\top \,\,(P_r \, h \, P_c) \\
     &= P_c^\top \, h^\top \, P_r^\top \, P_r \, h \, P_c \\
     &= P_c^\top \, h^\top \, h \, P_c \,.
   \end{align}
   Since $P_c^\top \, h^\top h \, P_c$ is a similarity transform of $h^\top h$, which does not change the eigenvalues, $h^\top h$ and $(h')^\top h'$ have the same eigenvalues, and in turn $h$ and $h'$ share the same singular values. hence $\sigma(h')=\sigma(h)$, and therefore
   \begin{equation}
     d_p(h,h')=\| \sigma(h) - \sigma(h')\|_p =0\,,
   \end{equation}
   proving that, for any row or column permutation of $h$, the distance $d(h, h')$ remains unchanged.
\end{proof}

\subsection{Proof of Equivariance of Multi-Head Attention to Structured Permutations~\ref{thm:attn-equiv}}
\label{app:proof-attn-equiv}
\begin{proof}
We provide a detailed, step-by-step proof showing that our two-stage alignment procedure—inter-head reordering followed by intra-head permutations—preserves the functionality of a multi-head self-attention layer.
Let:
\begin{itemize}
    \item $X \in \mathbb{R}^{S \times d_{\text{model}}}$ be the input sequence.
    \item $W_q, W_k, W_v \in \mathbb{R}^{d_{\text{model}} \times d_{\text{model}}}$ be the query, key, and value projection matrices.
    \item $H$ be the number of attention heads, each of dimensionality $d_k = d_{\text{model}} / H$.
\end{itemize}

Define:
\begin{equation}
Q = X W_q, \quad K = X W_k, \quad V = X W_v,
\end{equation}
and split them by head:
\begin{equation}
Q = [Q_1, Q_2, \ldots, Q_H], \quad Q_i \in \mathbb{R}^{S \times d_k},
\end{equation}
and similarly for $K$ and $V$.
Let $P_{\text{inter}}$ be an inter-head permutation in $\mathcal{S}_H$, with induced permutation vector $\pi$, and let $P_{\text{intra}}^{(i)} \in \mathcal{S}_{d_k}$ be the intra-head permutation for head $i$. We form the block-permutation matrix:
\begin{equation}
P_{\text{attn}} = \sum_{i=1}^{H} E^{i,\pi(i)} \otimes P^{(i)}_{\text{intra}},
\end{equation}
where $E^{i,\pi(i)}$ is a binary $H\times H$ matrix with a single 1 at $(i,\pi(i))$, and $\otimes$ denotes the Kronecker product.

\subsection*{Step 1: Permuting the projection weights}
Applying $P_{\text{attn}}$ to the query projections gives:
\begin{align*}
Q' &= X W_q P_{\text{attn}} = Q P_{\text{attn}} \\
&= \left[ \sum_{j=1}^{H} Q_j P_{\text{attn}}[j,i] \right]_{i=1}^{H} \\ 
&= \left[ Q_{\pi^{-1}(i)} P^{\pi^{-1}(i)}_{\text{intra}} \right]_{i=1}^{H}
 \end{align*}
Hence,
\begin{equation}
Q'_i = Q_{\pi^{-1}(i)}\,P^{\pi^{-1}(i)}_{\text{intra}},
\end{equation}
where the new head $Q'_i$ corresponds to the head designated by the inter-head permutation $\pi^{-1}(i)$, modified according to $P_{\text{intra}}^{\pi^{-1}(i)}$. The same applies to:
\begin{equation}
K'_i = K_{\pi^{-1}(i)}P^{\pi^{-1}(i)}_{\text{intra}}, \quad V'_i = V_{\pi^{-1}(i)}P^{\pi^{-1}(i)}_{\text{intra}}.
\end{equation}

\subsection*{Step 2: Permuting the attention scores}
Because each $P_{\text{intra}}^{(i)}$ is orthogonal ($P P^T = I$), the attention scores satisfy:
\begin{align*}
A'_i &= \mathrm{softmax}\left(\frac{Q'_i {K'_i}^T}{\sqrt{d_k}}\right) \\
     &= \mathrm{softmax}\left(\frac{Q_{\pi^{-1}(i)} P_{\text{intra}}^{\pi^{-1}(i)} (P_{\text{intra}}^{\pi^{-1}(i)})^T K_{\pi^{-1}(i)}^T}{\sqrt{d_k}}\right) \\
     &= \mathrm{softmax}\left(\frac{Q_{\pi^{-1}(i)} K_{\pi^{-1}(i)}^T}{\sqrt{d_k}}\right) = A_{\pi^{-1}(i)}.
\end{align*}
Thanks to the orthogonality of the intra-head permutation blocks, the attention scores are only influenced by the inter-head permutation.

\subsection*{Step 3: Permuting the value outputs}
For each head,
\begin{equation}
O'_i = A'_i V'_i = A_{\pi^{-1}(i)} V_{\pi^{-1}(i)} P_{\text{intra}}^{\pi^{-1}(i)} = O_{\pi^{-1}(i)} P_{\text{intra}}^{\pi^{-1}(i)}.
\end{equation}

\subsection*{Step 4: Reconstructing the final output}
Concatenating all heads yields:
\begin{equation}
O' = [O'_1, O'_2, \ldots, O'_H] = O P_{\text{attn}}.
\end{equation}

\paragraph{Conclusion.}
Applying $P_{\text{attn}}$ to the projection matrices is thus equivalent to permuting the output of multi-head attention. The self-attention layer remains functionally equivalent, and the original output can be recovered via $O = O'P_{\text{attn}}^T$.
\end{proof}

\subsection{Full Procedure to Manage Residual Connections}
\label{proof:residual}
\begin{proof}
\small
We begin with the standard formulation of a transformer block, ignoring LayerNorm for simplicity:
\begin{equation}
\begin{aligned}
\mathbf{z}_{attn} &= W_{0}\operatorname{MHA}(\mathbf{x}), \\
\mathbf{z}_i &= \mathbf{z}_{attn} + \mathbf{x}, \\ 
\mathbf{z}_f &= W_2 \text{ReLU}(W_1 \mathbf{z}_i), \\
\mathbf{z}_{out} &= \mathbf{z}_f + \mathbf{z}_i.
\end{aligned}
\end{equation}
Ignoring the ReLU activation function as well, we examine the impact of applying a permutation to one layer within a transformer block and then reversing it in the subsequent layer. This transformation leads to:
\begin{equation}
\begin{aligned}
\mathbf{z}_{\text{attn}} &= {P}_{\mathrm{W_0}}{W}_0 {P}_{\mathrm{attn}}^\top \Bigr( {P}_{\mathrm{attn}} \operatorname{MHA} {P}_{\mathrm{in}}^\top ({P}_{\mathrm{in}} \mathbf{x}) \Bigr), \\
\mathbf{z}_i &= {P}_{\mathrm{W_0}} \mathbf{z}_{\text{attn}} + {P}_{\mathrm{in}} \mathbf{x}, \\ 
\mathbf{z}_f &= {P}_{\mathrm{W_2}} {W}_2 {P}_{\mathrm{W_1}}^\top \Bigr( {P}_{\mathrm{W_1}} {W}_1 {P}_{\mathrm{W_0}}^\top (\mathbf{z}_i)\Bigr), \\
\mathbf{z}_{\text{out}} &= {P}_{\mathrm{W_2}} \mathbf{z}_f + {P}_{\mathrm{W_0}} \mathbf{z}_i.
\end{aligned}
\end{equation}
To ensure consistency in the permutation applied to both addends within each residual block, we replace the identity mapping with a permutation composition \(\mathcal{I}\), where \(\mathcal{I}_i = {P}_{\mathrm{W_0}} {P}_{\text{in}}^\top\) and \(\mathcal{I}_{\text{out}} = {P}_{\mathrm{W_2}} {P}_{\mathrm{W_0}}^\top\). This results in:

\begin{equation}
\begin{aligned}
\mathbf{z}_i &= {P}_{\mathrm{W_0}} \mathbf{z}_{\text{attn}} + \mathcal{I}_i{P}_{\text{in}}\mathbf{x} 
= {P}_{\mathrm{W_0}} \mathbf{z}_{\text{attn}} + {P}_{\mathrm{W_0}} \mathbf{x}, \\ 
\mathbf{z}_{\text{out}} &= {P}_{\mathrm{W_2}} \mathbf{z}_f + \mathcal{I}_{\text{out}}{P}_{\mathrm{W_0}} \mathbf{z}_i  
= {P}_{\mathrm{W_2}} \mathbf{z}_f + {P}_{\mathrm{W_2}} \mathbf{z}_i.
\end{aligned}
\end{equation}

After incorporating these compositions, the permutations remain consistent across each residual path, simplifying the block equations to:
\begin{equation}
\begin{aligned}
\mathbf{z}_{\text{attn}} &= {P}_{\mathrm{W_0}} {W}_0 \Bigr( \operatorname{MHA} (\mathbf{x}) \Bigr), \\
\mathbf{z}_i &= {P}_{\mathrm{W_0}} \mathbf{z}_{\text{attn}} + {P}_{\mathrm{W_0}} \mathbf{x} , \\ 
\mathbf{z}_f &= {P}_{\mathrm{W_2}} {W}_2 \Bigr({W}_1 (\mathbf{z}_i)\Bigr), \\
\mathbf{z}_{\text{out}} &= {P}_{\mathrm{W_2}} \mathbf{z}_f + {P}_{\mathrm{W_2}} \mathbf{z}_i.
\end{aligned}
\end{equation}
With ${P}_{\mathrm{W_2}}$ serving as the input permutation for the subsequent layer.
\end{proof}

\subsection{Extended Comparison on the Application of the Task Vector}
\label{sec:app_other_datasets}
In~\cref{fig:alpha_sup}, we extend the sensitivity analysis of the scaling coefficient $\alpha$ --- originally presented in~\cref{fig:alpha} --- to additional visual classification tasks.
\begin{figure}[t]
\centering
\includegraphics[width=.48\textwidth]{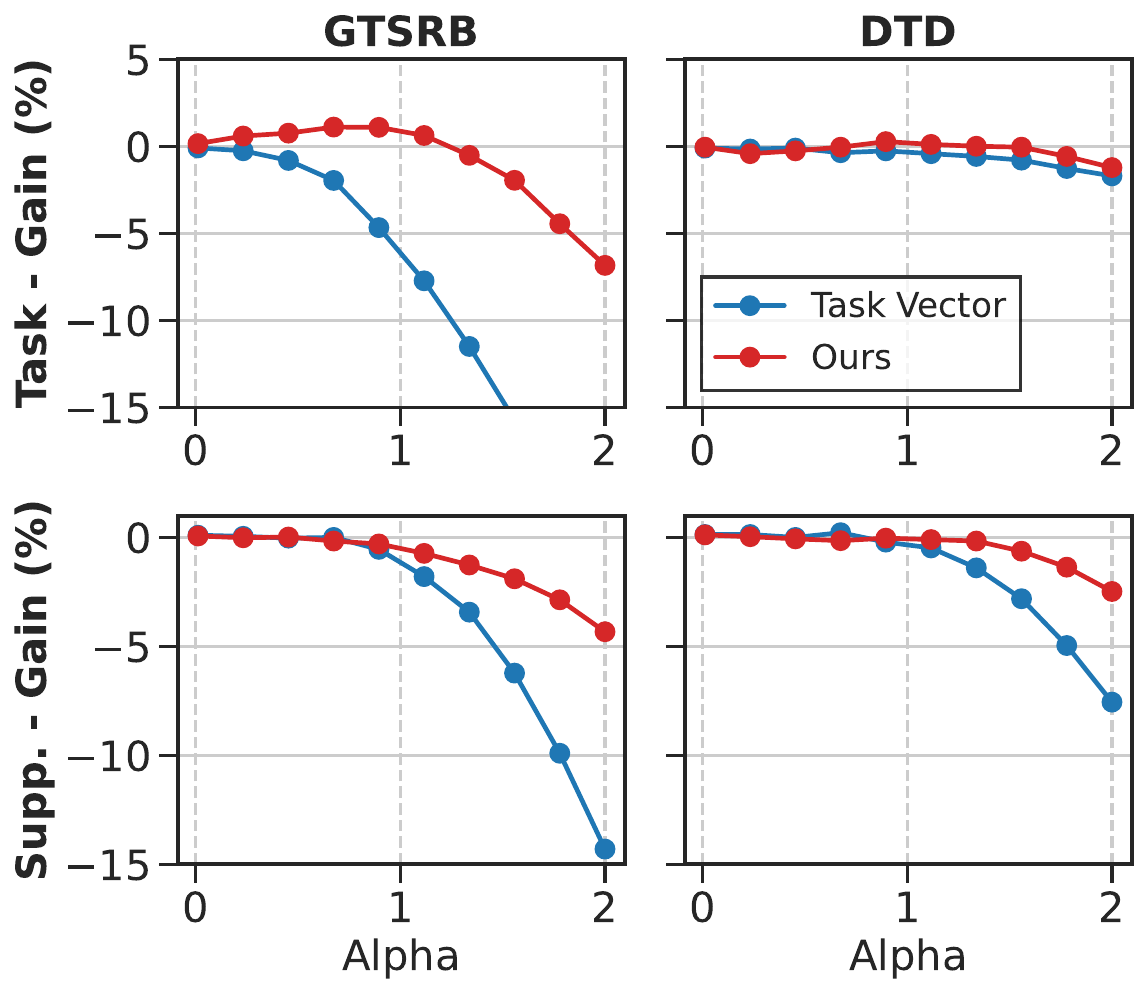} 
\caption{Zero-shot gain/drop relative to $\theta_B$ of naive $\theta_B+\alpha\tau$ (\textcolor{blue}{blue}) and our strategy $\theta_B+\alpha\pi(\tau)$ (\textcolor{red}{red}) varying $\alpha$.}
\label{fig:alpha_sup}
\end{figure}

\subsection{Proof of Proposition~\ref{prop:complexity}}
\label{app:complexity}
\begin{proposition}
\begin{proof}
To assess how computational complexity scales with model size, we define:

\begin{itemize}
    \item $L$: number of layers, evenly divided into MLP ($\frac{L}{2}$) and self-attention ($\frac{L}{2}$).
    \item $H$: number of attention heads.
    \item Each MLP layer contains two linear projections with dimension $(d_m, d_h)$ and $(d_h, d_m)$, assuming $d_m = d_h$.
    \item Self-attention layers have Q, K, and V matrices, each of size $(d_m, d_m)$.
\end{itemize}

We now estimate the complexity for a single iteration of the weight-matching algorithm.

\paragraph{MLP Layers.}
The main computational cost comes from computing a pairwise similarity matrix between rows of projection matrices ($O(d_m^3)$), and solving a $(d_m, d_m)$ assignment via Hungarian algorithm ($O(d_m^3)$). Hence, per-layer cost is:
\begin{align}
    O(d_m^3)
\end{align}

\paragraph{Self-Attention Layers.}
Split into two steps: inter-head and intra-head permutation.

\textbf{Inter-head permutation:}
\begin{itemize}
    \item $6H$ SVDs over matrices of size $(\frac{d_m}{H}, d_m)$
    \begin{align}
        \quad O\left(\frac{6d_m^3}{H}\right)
    \end{align}
    \item Distance matrix over heads:
    \begin{align}
        O\left(\frac{3H^2 d_m}{2}\right)
    \end{align}
    \item Hungarian algorithm over $(H, H)$ matrix:
    \begin{align}
        O(H^3)
    \end{align}
\end{itemize}

\textbf{Intra-head permutation:}
\begin{itemize}
    \item Per-head similarity:
    \begin{align}
        O\left( \left(\frac{d_m}{H}\right)^2 d_m \right)
        = O\left(\frac{d_m^3}{H^2} \right)
    \end{align}
    \item Hungarian algorithm per head:
    \begin{align}
        O\left( \left(\frac{d_m}{H}\right)^3 \right)
    \end{align}
\end{itemize}

Summed over $H$ heads:
\begin{equation}
    O\left( H \left( \frac{d_m^3}{H^2} + \left(\frac{d_m}{H}\right)^3 \right) \right) = O\left( \frac{d_m^3}{H} + \frac{d_m^3}{H^2} \right)
\end{equation}

\paragraph{Total Self-Attention Cost per Layer.}
\begin{equation}
    O\Bigg( \frac{6 d_m^3}{H} + \frac{3H^2 d_m}{2} + H^3 + \frac{d_m^3}{H} + \frac{d_m^3}{H^2} \Bigg)
\end{equation}

\paragraph{Final Complexity.}
Summing across $\frac{L}{2}$ MLP and $\frac{L}{2}$ attention layers:
\begin{equation}
    O\Bigg( \frac{L}{2} d_m^3 + \frac{L}{2} \Big( \frac{6 d_m^3}{H} 
    + \frac{3H^2 d_m}{2} + H^3 + \frac{d_m^3}{H} + \frac{d_m^3}{H^2} \Big) \Bigg).
\end{equation}

This expression can be algebraically simplified to a more compact equivalent form:
\begin{align}
    O\left( L \left( d_m^3 + \frac{d_m^3}{H} 
    + \frac{d_m^3}{H^2} + H^3 + H^2 d_m \right) \right).
\end{align}

So, the complexity scales polynomially with $d_m$ and $H$, and remains significantly lower than data-based fine-tuning.
\end{proof}
\end{proposition}

\end{document}